%% file: main.tex
\documentclass[a4paper]{article} 
\usepackage[margin=1in]{geometry}

\usepackage{microtype}
\usepackage{graphicx}
\usepackage{subcaption}
\usepackage{booktabs} 
\usepackage{here}
\usepackage{authblk}

\usepackage{amsmath,amsthm,amsfonts,amssymb}
\usepackage{mathtools}
\usepackage{xcolor}
\usepackage[sort]{natbib}
\usepackage{algorithm} 
\usepackage{algpseudocode} 
\usepackage{url}
\usepackage{multirow} 

\usepackage[T1]{fontenc}
\usepackage[utf8]{inputenc}
\usepackage{bbm}

\usepackage{makecell}
\usepackage{thmtools}
\usepackage{thm-restate}
\usepackage[hypertexnames=false]{hyperref}

\input{macros}

\title{Improved Regret Analysis \\for Parallel Gaussian Process Bandit Optimization}

\date{}

\author[1]{Shion Takeno}
\author[2]{Shogo Iwazaki}

\affil[1]{Nagoya University}
\affil[2]{MI-6 Ltd.}

\affil[ ]{\texttt{{takeno.s.mllab.nit@gmail.com}}}
\affil[ ]{\texttt{{shogo.iwazaki@gmail.com}}}

\begin{document}
\maketitle

\begin{abstract}
This paper studies the regret analysis for parallel Gaussian process (GP) bandit optimization.
The known regret upper bounds for the widely used GP batched upper confidence bound and GP batched Thompson sampling (GP-BTS) suffer from a multiplicative factor with respect to the batch size $Q$.
To avoid this degradation, existing analyses require a polynomial number of uncertainty sampling (US) for $Q$ at the beginning of optimization.
However, this initial US phase is often ineffective in practice.
This paper shows that the regret upper bound without the multiplicative factor on $Q$ can be achieved without the initial US phase, using GP-BTS as an example.
Furthermore, we show much better regret upper bounds in the noiseless setting than in the noisy setting, as in the sequential GP bandit setting.
\end{abstract}
\section{Introduction}
\label{sec:intro}

Gaussian process (GP) bandit optimization \citep{Chowdhury2017-on,Srinivas2010-Gaussian} has emerged as a powerful framework for black-box optimization of expensive-to-evaluate functions, which is required in various scientific and engineering domains.
For this purpose, GP bandit methods model the unknown objective function by a GP and sequentially select evaluation points.
However, in many practical scenarios, such as high-throughput screening or asynchronous distributed computing, function evaluations can be performed in parallel or subject to delayed feedback \citep{hernandez-lobato2017parallel,Kandasamy2018-Parallelised}.
In such parallel and delayed settings, the algorithm must select new query points without waiting for the results of recent queries, while maintaining query point diversity.

For this problem, batched and parallel variants of GP bandit algorithms have been proposed.
A fundamental theoretical challenge in these settings is the degradation of the cumulative regret bound by a multiplicative factor of the batch size or delay parameter $Q$.
For instance, the regret bounds of GP batched upper confidence bound (GP-BUCB) \citep{chowdhury2019batch,Desautels2014-Parallelizing} and GP batched Thompson Sampling (GP-BTS) \citep{chowdhury2019batch,vakili2021-scalable} suffer from a multiplicative factor on $Q$.
This can lead to substantial degradation, particularly with massive parallelization ($Q > 100$) \citep{hernandez-lobato2017parallel,vakili2021-scalable}.
To avoid this issue, the known analysis by \citet{Desautels2014-Parallelizing} requires an initial phase of uncertainty sampling (US), which degrades the practical optimization performance.

In this paper, we resolve the above problem by an improved regret analysis technique and apply it to GP-BTS as an example.
We summarize our main contributions as follows:
\begin{itemize}
    \item We develop Lemma~\ref{lemInflatedEllipticalPotentialLemmaRevised}, which plays a key role in decoupling the batch size $Q$ as the additional term from the regret upper bound.
    \item We prove that GP-BTS achieves a cumulative regret bound in the noisy setting with only an additive degradation of $Q$, eliminating the need for the practically inefficient initial US phase, as summarized in Table~\ref{tab:summary_noisy}.
    \item We provide the regret upper bounds for GP-BTS in the noiseless setting by leveraging the proof technique of \citep{iwazaki2025gaussian}, demonstrating that the regret upper bounds are much better than in the noisy case, as summarized in Table~\ref{tab:summary_noiseless}.
\end{itemize}

\begin{table}[t]
    \caption{
        Summary of cumulative regret upper bounds for GP-BTS and GP-TS in the noisy setting with respect to the number of function evaluations $T$ and the batch size $Q$.
        %
        %
        Note that $\widetilde{O}$ suppresses polylogarithmic factors for $Q$ and $T$.
        See Theorem~\ref{TheoCumulativeRegretBoundNoisy} for the details and the simple regret bounds.
        }
    \centering
    \begin{tabular}{ccc}
        \toprule
        & SE & Mat\'ern \\ \hline \addlinespace
        GP-BTS (Ours) & $O\left( Q \ln^{d+1} Q \! + \! \sqrt{T \ln^{2d + 3} T} \right)$ & $\widetilde{O}\left( Q^{\frac{2 \nu + d}{2 \nu}} \! + T^{\frac{2\nu + 3d}{4\nu + 2d}} \right)$ \\ 
        \makecell{GP-BTS\\(Based on \citep{vakili2021-scalable})} & $O\left( \sqrt{Q T \ln^{d + 1} \frac{T}{Q}  \ln^{d + 2} T } \right)$ & $\widetilde{O}\left( Q^{\frac{\nu}{2\nu + d}} T^{\frac{2\nu + 3d}{4\nu + 2d}} \right)$ \\ 
        sequential GP-TS \citep{Chowdhury2017-on} & $O\left( \sqrt{T \ln^{2d + 3} T} \right)$ & $\widetilde{O}\left(T^{\frac{2\nu + 3d}{4\nu + 2d}} \right)$ \\
        \bottomrule
    \end{tabular}
    \label{tab:summary_noisy}
\end{table}

\begin{table}[t]
    \caption{
        Summary of cumulative regret upper bounds of GP-BTS, GP-UCB-based batched method (GP-UCB-batch), and sequential GP-UCB in the noiseless setting with respect to the number of function evaluations $T$ and the batch size $Q$.
        Note that $\widetilde{O}$ suppresses polylogarithmic factors for $Q$ and $T$.
        %
        %
        %
        %
        See Theorem~\ref{thm:CumulativeRegretUpperBoundNoiseless} for the details.
        }
    \centering
    \setlength{\tabcolsep}{1.5pt}
    \makebox[\textwidth][c]{
        \begin{tabular}{lcccc}
            \toprule
            & & \multicolumn{3}{c}{Mat\'ern} \\
            \cmidrule(lr){3-5}
            Method & SE & $d > \nu$ & $d = \nu$ & $d < \nu$ \\
            \midrule
            GP-BTS (Ours) 
            & $O(Q\ln^{\frac{1}{2}} T)$ 
            & $\widetilde{O}\left(Q + Q^{\frac{\nu}{d}} T^{\frac{d-\nu}{d}} \right)$ 
            & $O\left(Q \ln^{\frac{5}{2}} T \right)$ 
            & $O\left(Q \ln^{\frac{1}{2}} T\right)$ \\
            \addlinespace
            GP-UCB-batch \citep{lyu2019efficient}
            & $O\left( \! \sqrt{Q T \ln^{2d + 3} T} \right)$ 
            & \multicolumn{3}{c}{$\widetilde{O}\left( Q^{\frac{1}{2}} T^{\frac{2\nu + 3d}{4\nu + 2d}} \right)$} \\
            \addlinespace
            sequential GP-UCB \citep{iwazaki2025gaussian}
            & $O(1)$ 
            & $\widetilde{O}\left( T^{\frac{d-\nu}{d}} \right)$ 
            & $O\left( \ln^2 T \right)$ 
            & $O(1)$ \\
            \bottomrule
        \end{tabular}
        }
    \label{tab:summary_noiseless}
\end{table}

\subsection{Related work}
\label{sec:related}

\paragraph{Sequential GP bandits.}
In the sequential setting, the theoretical properties of GP bandits have been extensively studied under the frequentist assumption that the objective function belongs to a reproducing kernel Hilbert space (RKHS).
In particular, GP-UCB \citep{Chowdhury2017-on,Srinivas2010-Gaussian} and GP-TS \citep{Chowdhury2017-on} achieve the sublinear cumulative regret bounds.
Several algorithms, such as phased elimination (PE) and SupKernelUCB, achieve near-optimal regret bounds but are often impractical or computationally complex \citep{calandriello2019gaussian,iwazaki2025-improvedGPbandit,janz2020-bandit,li2022gaussian,valko2013finite,vakili2021-optimal}.
In particular, although PE can readily yield a near-optimal regret bound in the parallel setting, we focus on the analysis of GP-BTS due to PE's practical ineffectiveness.

\paragraph{Parallel GP bandits.}
The regret bounds of parallel GP bandits have been explored through various approaches.
GP-BUCB \citep{Desautels2014-Parallelizing,chowdhury2019batch} and GP-BTS \citep{chowdhury2019batch,vakili2021-scalable} suffer from the multiplicative factor on $Q$ in the regret upper bound.
Although \citet{Contal2013-Parallel} claimed that GP-UCB with pure exploration (GP-UCB-PE) achieves the cumulative regret upper bound without the multiplicative factor on $Q$, its proof leaves some technical ambiguities, \footnote{In the proof of Lemma~6 in \citep{Contal2013-Parallel}, the unjustified inequality $\hat{\sigma}_t(\*x^*) \leq \hat{\sigma}_t(\*x_t^k)$ seems to be used.} making the exact regret behavior unclear.
Similarly, under Bayesian assumptions, several regret upper bounds, which also suffer from a multiplicative factor on $Q$, have been shown \citep{Desautels2014-Parallelizing,Kandasamy2018-Parallelised,nava2022diversified,sugiura2026randomized}.
Our work resolves this fundamental bottleneck by achieving an additive factor on $Q$ without any initial US phase.

\paragraph{Stochastically delayed GP bandits.}
Recently, \citet{verma2022bayesian} and \citet{vakili2023delayed} have studied GP bandits for the stochastic delayed-feedback setting, based on a tailor-made confidence interval for this setting.
Although we focus on the deterministically delayed-feedback setting, an application of our analysis to this setting can be an important future direction.

\paragraph{GP bandits with noiseless feedback.}
The regret analyses are also conducted on the noiseless setting~\citep{bull2011convergence,lyu2019efficient,vakili2022open,salgia2024random,flynn2025-tighter,iwazaki2025-improvedGPbandit,iwazaki2025gaussian,kim2025enhancing}.
In particular, although \citet{lyu2019efficient} have studied the parallelization of GP bandit, their regret upper bounds have almost the same order as in the noisy case.
Recently, near-optimal cumulative regret bounds for the sequential GP-UCB have been derived by \citet{iwazaki2025gaussian}.
We obtain much better regret bounds than \citep{lyu2019efficient} by leveraging the proof of \citep{iwazaki2025gaussian}.

\paragraph{Regret lower bound.}
Several studies have provided algorithm-independent regret lower bounds for the GP bandit problem \citep{cai2021on,li2024regret,ma2026batched,scarlett2017lower}.
For both noisy and noiseless settings, our regret analysis is not tight, since PE and US-based algorithms can achieve the near-optimal regret bound that matches the regret lower bound therein (see \citep{iwazaki2025gaussian,iwazaki2025-improvedGPbandit}).
On the other hand, since PE and US-based algorithms are often practically ineffective, we tackle tightening the regret upper bounds of the GP-TS-based parallelization method, which is widely used \citep{chowdhury2019batch,hernandez-lobato2017parallel,Kandasamy2018-Parallelised,nava2022diversified,vakili2021-scalable}.

\section{Preliminaries}
\label{sec:prelimary}

This section presents the problem statement, surrogate models, and assumptions.

\subsection{Problem statement}

We consider the $d$-dimensional maximization problem of $f: \cX \rightarrow \RR$ with an input domain $\cX \subset \RR^d$:
\begin{align*}
    \*x^* \in \argmax_{\*x \in \cX} f(\*x),
\end{align*}
where the evaluations of the most recent $Q$ queries are not yet observed.
Since evaluating the objective function $f$ is assumed to be expensive, we sequentially choose $\*x_t$ for all iterations $t$ and obtain a (possibly noisy) observation $y_t$, by which we aim to optimize with a smaller number of observations.
Hence, the dataset that we have at the $t$-th iteration is $\cD_{t-Q-1}= \cbr{(\*x_i, y_i)}_{i=1}^{t-Q-1}$ since we do not have the observations at $\*x_{t-Q}, \dots, \*x_{t-1}$.
Thus, setting $Q=0$ recovers the usual sequential optimization problem.
This setting includes optimization problems with delayed or parallel feedback, a typical example of which is an asynchronous parallelization setup with $Q + 1$ workers.
Note that although we consider this delayed setting for simplicity, our analysis can apply to other settings, such as the batch setting where $Q + 1$ queries are issued simultaneously.

\paragraph{Regret.}
We evaluate the performance by the cumulative and simple regrets \citep{Srinivas2010-Gaussian,Russo2014-learning} defined below:
\begin{align*}
    R_T = \sum_{t = 1}^T f(\*x^*) - f(\*x_t),
    \quad
    r_T = f(\*x^*) - f(\hat{\*x}_T),
\end{align*}
where $\hat{\*x}_T$ is some recommendation input by the algorithm at the end of the $T$-th iteration.
As detailed in Section~\ref{sec:analysis}, we will analyze the expectation of the regret, where the expectation is taken with respect to all the randomness arising from the algorithm and the noise sequence (if it exists).

\subsection{Gaussian process regression (GPR) model}
\label{sec:GPR}

We use a GPR model \citep{Rasmussen2005-Gaussian} as the surrogate model, as with \citep{Chowdhury2017-on,janz2020-bandit,Srinivas2010-Gaussian}.
Suppose that we have the training dataset $\cD_t = \{ (\*x_i, y_i) \}_{i=1}^t$, where $y_t = f(\*x_t) + \epsilon_t$ with i.i.d. noise $\epsilon_t \sim \cN(0, \lambda^2)$.
Under the assumption that $f$ follows a GP denoted as $\cG \cP (0, k)$, where $k: \cX \times \cX \rightarrow \RR$ is a kernel function, the posterior distribution $p(f \mid \cD_t)$ is a GP again.
That is, $f \mid \cD_t \sim \cG\cP(\mu_{t}, k_{t} )$, whose posterior mean and covariance functions are given below:
\begin{equation}
    \begin{split}
        \mu_{t}(\*x) &= \*k_{t}(\*x)^\top \bigl(\*K_t + \lambda^2 \*I_{t} \bigr)^{-1} \*y_{t}, \\
        k_{t} (\*x, \*x^\prime) &= k(\*x, \*x^\prime) - \*k_{t}(\*x)^\top \bigl(\*K_t + \lambda^2 \*I_{t} \bigr)^{-1} \*k_{t}(\*x^\prime),
    \end{split}
    \label{eq:GP}
\end{equation}
where $\*k_{t}(\*x) = \bigl( k(\*x, \*x_1), \dots, k(\*x, \*x_{t}) \bigr)^\top \in \RR^{t}$, $\*K_t \in \RR^{t \times t}$ is the kernel matrix whose $(i, j)$-element is $k(\*x_i, \*x_j)$, $\*I_{t} \in \RR^{t \times t}$ is the identity matrix, and $\*y_{t} = (y_1, \dots, y_t)^\top \in \RR^{t}$.
For the sake of notational simplicity, we define 
$\sigma_t^2(\*x) = k_t(\*x, \*x)$ and
$\mu_{-i}(\*x) = 0$, 
$\sigma_{-i}^2 (\*x) = k(\*x, \*x)$,
and $k_{-i} (\*x, \*x^\prime) = k(\*x, \*x^\prime)$
for all $i = \{0, \dots, Q \}$ and $\*x, \*x^\prime \in \cX$. 
When $\lambda = 0$ and $\*K_t$ becomes not invertible though $\*K_{t-1}$ is invertible, we do not add $(\*x_t, y_t)$ to the training dataset, that is, update $\cD_t \gets \cD_{t-1}$ as in \citep{iwazaki2025gaussian}.

\paragraph{Maximum information gain (MIG).}
The complexity of GP bandits is commonly quantified by the MIG~\citep{iwazaki2025improved,iwazaki2026tighter,Srinivas2010-Gaussian,vakili2021-information} defined below:
\begin{definition}
    Let $f \sim \cG \cP (0, k)$ over a compact input domain $\cX \subset [0, r]^d$.
    %
    %
    Then, MIG $\gamma_T$ is defined as follows:
    \begin{align*}
        \gamma_T(\lambda^2) 
        = \frac{1}{2} \sup_{\*x_1, \dots, \*x_T \in \cX}  \ln\left( {\rm det} \left( \*I_T + \lambda^{-2} \*K_T \right) \right),
    \end{align*}
    where $\*K_T \in \RR^{T \times T}$ is the kernel matrix whose $(i, j)$-element is $k(\*x_i, \*x_j)$.
    \label{def:MIG}
\end{definition}
The MIG of several commonly used kernel functions is known to be sublinear regarding $T$.
For example, $\gamma_T (\lambda^2) = O\bigl( \bigl(\ln (T / \lambda^2)\bigr)^{d+1} \bigr)$ for SE kernels $k_{\rm SE} (\*x, \*x^\prime) = \exp\left( - \| \*x - \*x^\prime \|_2^2 / (2 \ell^2) \right)$, and $\gamma_T (\lambda^2) = O\bigl( (T / \lambda^2)^{\frac{d}{2\nu + d}} \ln^{\frac{2\nu}{2\nu + d}} (T / \lambda^2) \bigr) = \widetilde{O}\bigl( (T / \lambda^2)^{\frac{d}{2\nu + d}} \bigr)$ for Mat\'{e}rn-$\nu$ kernels $k_{\rm Mat}(\*x, \*x^\prime) = \frac{2^{1 - \nu}}{\Gamma(\nu)} \left( \frac{\sqrt{2\nu} \| \*x - \*x^\prime \|_2 }{\ell} \right)^{\nu} J_{\nu} \left( \frac{\sqrt{2\nu} \| \*x - \*x^\prime \|_2 }{\ell} \right) $, where $\ell, \nu > 0$ are the lengthscale and smoothness parameter, respectively, and $\Gamma(\cdot)$ and $J_{\nu}$ are Gamma and modified Bessel functions, respectively \citep{iwazaki2025improved,iwazaki2026tighter,Srinivas2010-Gaussian,vakili2021-information}.
Note that $\widetilde{O}$ suppresses polylogarithmic factors.
\begin{remark}
    We analyze regret via the widely used MIG bound from \citep{vakili2021-information}.
    However, its flaw for Mat\'ern kernels has been pointed out by \citet{janz_2021} and \citet{iwazaki2025improved}. 
    Theorem 7 in \citet{iwazaki2025improved} provides the rectified MIG bound for the Mat\'ern kernels $\gamma_T = O\bigl( (T / \lambda^2)^{\frac{d}{2\nu + d}} \ln^{\frac{4\nu + d}{2\nu + d}} (T / \lambda^2) \bigr)$, which requires the additional polylogarithmic factor on $T$.
    In addition, a slightly tighter MIG bound for the SE kernels is provided by \citet{iwazaki2026tighter}. 
    By adapting these analyses, our regret upper bounds may change slightly.
\end{remark}
%
%
Hereafter, when we focus only on the order except for $\lambda^2$, we denote $\gamma_t (\lambda^2) = \gamma_t$ for simplicity.

\subsection{Regularity assumptions}
\label{sec:assumptions}

Although we use the GPR model defined in Sec. \ref{sec:GPR}, we assume a condition called the frequentist setting, which is different from the Bayesian assumption of the GPR model.
We assume the following common regularity assumption on $f$ \citep{Chowdhury2017-on,iwazaki2025-improvedGPbandit,janz2020-bandit,Srinivas2010-Gaussian}:
\begin{assumption}
    Let $f$ be an element of RKHS $\cH_k$ specified by the predefined kernel $k$ that satisfies $\forall \*x \in \cX, k(\*x, \*x) \leq 1$, where $\cX \subset [0, r]^d$ is a compact input space.
    Furthermore, the RKHS norm of $f$ is bounded as $\| f \|_{\cH_k} \leq B < \infty$ for some $B > 0$, where $\| \cdot \|_{\cH_k}$ denotes the RKHS norm of $\cH_k$.
    \label{assump:function}
\end{assumption}
This assumption imposes desirable properties on $f$, such as smoothness and boundedness.

We consider both the noisy case \citep{Russo2014-learning,Srinivas2010-Gaussian}, where the observations are contaminated by noise, and the noiseless case \citep{freitas2012exponential,iwazaki2025gaussian}, where we can obtain the function value as $y_t = f(\*x_t)$.
For the noisy case, we assume the following condition on the noise sequence as in \citep{Chowdhury2017-on,janz2020-bandit}:
\begin{assumption}
    We can observe contaminated observations $y_t = f(\*x_t) + \epsilon_t$ for all $t \in \NN$.
    The noise sequence $(\epsilon_t)_{t \in \NN}$ is a sequence of conditionally $R$-sub-Gaussian random variables.
    That is, for some $R > 0$, the moment generating function of $\epsilon_t$ satisfies $\EE [\exp (\eta \epsilon_t) \mid \cF_{t-1}] \leq \exp\bigl( \frac{\eta^2 R^2}{2} \bigr)$ for all $t \in \NN$ and $\eta \in \RR$, where $\cF_{t-1}$ is $\sigma$-algebra generated by the random variables $\{ \*x_i, y_i \}_{i=1}^{t-1}$ and $\*x_t$.
    \label{assump:noisy}
\end{assumption}

Furthermore, we assume the following conditions on smoothness and discretization as in \citep{Chowdhury2017-on,chowdhury2019batch,vakili2021-scalable,vakili2021-optimal,vakili2022improved}:
\begin{assumption}
    The kernel function $k$ satisfies the following condition on the derivatives.
    There exists a constant $L_k$ such that,
    $
        \sup_{\*x \in \cX} \sup_{j \in [d]} \left| \frac{\partial^2 k(\*u, \*v)}{\partial u_j \partial v_j}\bigg|_{\*u=\*v=\*x} \right|^{1/2} \leq L_k.
    $
    Furthermore, there exists a (computable) finite subset $\cX_t \subset \cX$ that satisfies $\max_{\*x \in \cX} \min_{\*x^\prime \in \cX_t} \| \*x - \*x^\prime \|_1 \leq \frac{1}{L_k N_t}$ and $|\cX_t| \leq (L_k drN_t)^d$ for all $t \in [T]$ and $N_t \in \NN$.
    \label{assump:discretization}
\end{assumption}
This assumption is satisfied at least by the SE kernels and the Mat\'ern kernels with $\nu > 1$.
The algorithm of GP-BTS depends on a predefined discretized input set $\cX_t \subset \cX$, as with \citep{Chowdhury2017-on,chowdhury2019batch,vakili2021-scalable}.
(For details, see Section~4.2 in \citep{Chowdhury2017-on}.)

\section{GP-BTS}
\label{sec:GP-BTS}

This section describes the algorithm and the existing analyses of GP-BTS.

\subsection{Algorithm}
\label{sec:algorithm}

We define GP-BTS following \citep{chowdhury2019batch,vakili2021-scalable}, whose pseudocode is shown in Algorithm~\ref{alg:GP-BTS}.
In our problem setup, since the recent $Q$ observations are not yet available, we need to query $\*x_t$ that is diverse relative to the recent $Q$ inputs.
From the nature of GP-TS, which depends on the random posterior sampling, we can determine $\*x_t$ that is randomly distributed, by which we expect that the inputs have diversity.
Thus, GP-BTS chooses the $t$-th input using the currently available dataset $\cD_{t - Q - 1}$ as follows:
\begin{align*}
    \*x_t = \argmax_{\*x \in \cX_t} g_t(\*x),
\end{align*}
where $g_t \sim \cG \cP(\mu_{t - Q - 1}, \beta_t k_{t - Q - 1})$ is the posterior sample path with the inflated posterior variance by the confidence width parameter $\beta_t$ to apply the proof technique by \citet{Chowdhury2017-on}.
Furthermore, we will theoretically specify $\beta_t$ as $\beta_t = \Theta(\gamma_t)$, as in \citep{vakili2021-scalable}, though GP-BTS in \citep{chowdhury2019batch} and GP-BUCB \citep{Desautels2014-Parallelizing,chowdhury2019batch} set $\beta_t = \Theta\bigl(\exp (\gamma_Q) \gamma_t\bigr)$.

\begin{algorithm}[!t]
    \caption{GP-BTS}\label{alg:GP-BTS}
    \begin{algorithmic}[1]
        \Require Domain $\cX$, kernel $k$, confidence width parameters $\{ \beta_t \}_{t \in \NN}$, discretized inputs $\{ \cX_t \}_{t \in \NN}$
        \State $\cD_{0} \gets \emptyset$
        \For{$t = 1, \dots, T$}
            \State If $t > Q + 1$, update $\mu_{t-Q-1} (\cdot)$ and $\sigma_{t-Q-1}^2 (\cdot)$ by Eq.~\eqref{eq:GP}
            \State Compute $\*x_t = \argmax_{\*x \in \cX_t} g_t(\*x)$ where $g_t \sim \cG \cP(\mu_{t - Q - 1}, \beta_t k_{t - Q - 1})$
            \State If $t \geq Q + 1$, observe $y_{t - Q} = f(\*x_{t - Q}) + \epsilon_{t - Q}$
        \EndFor
        \State \Return some recommendation $\hat{\*x}_T$
    \end{algorithmic}
\end{algorithm}

\subsection{Existing regret analysis}
\label{sec:existing_analysis}

Although \citet{vakili2021-scalable} focused on the GP-BTS with sparse GP approximation, their analysis can apply to our problem setup, where we can compute an exact GP posterior and an exact posterior sample path.
By adapting the analysis by \citet{vakili2021-scalable}, GP-BTS attains the following cumulative regret bound (a detailed and modified version of this claim is shown in Lemma~\ref{lem:GeneralRegretUpperBound}):
\begin{align}
    \EE\sbr{R_T}
    \lesssim \sqrt{\beta_T \ln T} \EE\sbr{\sum_{t=1}^T \sigma_{t-Q-1}(\*x_t)},
\end{align}
where $\lesssim$ ignores the constant factors except for $T$ and $Q$, and the expectation is taken with respect to the randomness of the algorithm and the noise sequence.
Since $\sigma_{t-Q-1}(\*x_t) \geq \sigma_{t-1}(\*x_t)$ holds, the above regret upper bound is larger than the regret bound for the sequential case, which depends on the quantity $\sum_{t=1}^T \sigma_{t-1}(\*x_t)$.
To obtain the upper bound for $\sum_{t=1}^T \sigma_{t-Q-1}(\*x_t)$, \citet{vakili2021-scalable} apply the following lemma:\footnote{Note that, due to the difference in the notation, $BT$ in \citep{vakili2021-scalable} is equivalent to $T$ in this paper.}
\begin{lemma}[Adapted from Lemma 3 in \citep{vakili2021-scalable}]
    \label{lem:VakiliLem3}
    For any input sequence $\*x_1, \dots, \*x_T \in \cX$, the following inequality holds:
    \begin{align*}
        \sum_{t=1}^T \sigma_{t-Q-1}(\*x_t)
        &\leq \sqrt{C_1 Q T \gamma_{T/Q}},
    \end{align*}
    where $C_1 = 2 / \ln(1 + \lambda^{-2})$.
\end{lemma}

Finally, \citet{vakili2021-scalable} have derived the following cumulative regret upper bound of GP-BTS:
\begin{align}
    \label{eq:CumualativeRegretBound_vakili}
    \EE\left[ R_T \right] = O\left( \sqrt{Q \gamma_{T / Q} \gamma_T T \ln T} \right).
\end{align}
Compared with the $O\rbr{\gamma_T \sqrt{T \ln T}}$ bound of the sequential GP-TS \citep{Chowdhury2017-on}, $\sqrt{\gamma_T}$ is replaced with $\sqrt{Q \gamma_{T / Q}}$.
Since we focus on the case where $\gamma_T$ is sublinear, this replacement is a degradation with respect to $Q$.
%
%
In the next section, we show that this degradation can be avoided without the impractical initial US phase.

\section{Improved regret analysis}
\label{sec:analysis}

%
%
First, we derive Lemma~\ref{lem:GeneralRegretUpperBound}, which shows the instantaneous regret upper bound of GP-BTS, modified slightly from \citep{vakili2021-scalable} to apply Lemmas~\ref{lemInflatedEllipticalPotentialLemmaRevised} and \ref{lemPosteriorSTDUpperBoundNoiseless_informal}.
Then, we show the regret upper bounds for GP-BTS in noisy and noiseless settings.

We can obtain the following instantaneous regret bound for GP-BTS:
\begin{restatable}[Modified from Theorem 1 in \citep{vakili2021-scalable}]{lemma}{lemGeneralRegretUpperBound} \label{lem:GeneralRegretUpperBound}
    Suppose Assumptions~\ref{assump:function} and \ref{assump:discretization} hold.
    Set $p_t^{(g)} \in (0, 1)$, $N_t \in \NN$ and 
    (i) $\lambda > 0$, $p_t^{(f)} \in (0, 1)$ and $\beta_t^{1/2}(\delta) = B + \frac{R}{\lambda} \sqrt{2\bigl(\gamma_t + \ln (1 / \delta)\bigr)}$ for the noisy setting where Assumption~\ref{assump:noisy} holds or
    (ii) $\lambda = p_t^{(f)} = 0$ and $\beta_t^{1/2}(\delta) = B$ for the noiseless setting where $y_t = f(\*x_t)$.
    Then, if Algorithm~\ref{alg:GP-BTS} runs, for all $t \in [T]$, the following inequality holds:
    \begin{align*}
        \EE[f(\*x^*) - f(\*x_t)] 
        \leq 
        \frac{2 + p}{p} \EE\left[ \min \left\{2B, c_t \sigma_{t-Q-1} (\*x_t)\right\} \right]
        + \frac{B}{N_t} 
        + 2B p_t^{(f)}
        + 2B p_t^{(g)}
    \end{align*}
    where 
    $p = 1 / (4 e \sqrt{\pi})$,
    $\zeta_t (\delta) = 2 \ln\bigl( 2 (L_k drN_t)^d / \bigl(p \delta \bigr) \bigr)$,
    $c_t = \beta_t^{1/2}\bigl(p_t^{(f)}\bigr) \left(2 \zeta_t^{1/2}\bigl(p_t^{(g)}\bigr) + 1 \right)$,
    and the expectation is taken with respect to $\{ \epsilon_t \}_{t \in [T]}$ and $\{ g_t \}_{t \in [T]}$ for the noisy setting and $\{ g_t \}_{t \in [T]}$ for the noiseless setting.
\end{restatable}
See Appendix~\ref{sec:proof_GeneralRegretUpperBoundNoisy} for the proof.

\paragraph{Comparison from \citep{vakili2021-scalable}.}
Our modification in Lemma~\ref{lem:GeneralRegretUpperBound} compared with \citep{vakili2021-scalable} is useful to bound the regrets over some index set $\cT$.
The difference is a replacement of $\min \cbr{2B, c_t \EE\left[\sigma_{t-Q-1} (\*x_t) \right]}$ with $\EE\left[ \min \left\{2B, c_t \sigma_{t-Q-1} (\*x_t)\right\} \right]$.
If we use the upper bound $\min \cbr{2B, c_t \EE\left[\sigma_{t-Q-1} (\*x_t) \right]}$ naively, the regret upper bound incurs the term $c_T |\cT| = O(\sqrt{\gamma_T \ln T} |\cT|)$, where $\cT$ is the set of iterations defined in Lemma~\ref{lemInflatedEllipticalPotentialLemmaRevised}.
We can utilize Lemma~\ref{lem:GeneralRegretUpperBound} to obtain $\sum_{t \in \cT} f(\*x^*) - f(\*x_t) = O(B |\cT|)$.

\subsection{Regret analysis for noisy feedback} 
\label{sec:NoisyAnalysis}

The degradation from the sequential setting arises mainly from Lemma~\ref{lem:VakiliLem3}.
To avoid this, we obtain the following lemma similar to the elliptical potential count lemma (Lemma D.9 in \citep{flynn2025-tighter}, Lemma 3.3 in \citep{iwazaki2025-improvedGPbandit}, Lemma 6 in \citep{iwazaki2025gaussian}):
\begin{restatable}{lemma}{lemInflatedEllipticalPotentialLemmaRevised}
    \label{lemInflatedEllipticalPotentialLemmaRevised}
    For any $T \in \NN$, $\lambda > 0$, and any input sequence $\{\*x_t \}_{t \in [T]}$, the following hold:
    \begin{align*}
        |\cT| \leq 6 Q \gamma_{|\cT|} (\lambda^2) + Q \leq 6 Q \gamma_{T} (\lambda^2) + Q, 
        \text{ and }
        \sum_{t \in \cT^c} \sigma_{t - Q - 1}^2(\*x_t) \leq 8 \lambda^2 \gamma_T,
    \end{align*}
    where $\cT = \{t \in [T] \mid \sigma_{t - Q - 1}^2(\*x_t) \geq \lambda^2 / Q \}$ and $\cT^c$ is its complement set.
\end{restatable}
This lemma can be derived by combining a proof similar to the elliptical potential count lemma with the posterior variance inequality shown in Lemma~\ref{lemPoteriorVarianceInequality}.
See Appendix~\ref{sec:proof_lemInflatedEllipticalPotentialLemma} for the detailed proof.
%
This lemma implies that, for most iterations $\cT^c$, where $|\cT^c| = T - |\cT| \geq T - 6 Q \gamma_{|\cT|} (\lambda^2) - Q$, we can apply the usual MIG upper bound without degradation with respect to $Q$.
Thus, we can see that, for $t \in \cT^c$, we can obtain a tighter upper bound compared with Lemma~\ref{lem:VakiliLem3}.

By combining Lemmas~\ref{lem:GeneralRegretUpperBound} and \ref{lemInflatedEllipticalPotentialLemmaRevised}, we can obtain the following regret upper bounds:
\begin{restatable}{theorem}{TheoCumulativeRegretBoundNoisy}
    \label{TheoCumulativeRegretBoundNoisy}
    Suppose Assumptions~\ref{assump:function}, \ref{assump:noisy}, and \ref{assump:discretization} hold.
    Set
    $1 / N_t = p_t^{(f)} = p_t^{(g)} = 1 / t^2$,
    $\hat{\*x}_T = \argmax_{\*x \in \cX} \bigl\{ \mu_{T} (\*x) - \beta_T^{(1/2)}(p_T^{(g)}) \sigma_T (\*x) \bigr\}$,
    and other variables as in Lemma~\ref{lem:GeneralRegretUpperBound}.
    Then, the following inequality holds:
    \begin{align*}
        \EE[R_T]
        &= O\left(\overline{T}^{(Q)} + \gamma_T \sqrt{T \ln T} \right), 
        \quad
        \EE[r_T]
        = O\left(\frac{\overline{T}^{(Q)} + \gamma_T \sqrt{T \ln T}}{T} \right),
    \end{align*}
    where $\overline{T}^{(Q)} = \max \{t \in \NN \mid t \leq 6 Q \gamma_t + Q \}$.
    In particular, 
    $\overline{T}^{(Q)} = O\rbr{Q \ln^{d+1} Q}$ for SE kernels and
    $\overline{T}^{(Q)} = \widetilde{O}\rbr{Q^{\tfrac{2\nu + d}{2\nu}} }$ for Mat\'ern kernels.
\end{restatable}
In the proof, roughly speaking, we show that $\EE\sbr{\sum_{t \in \cT} f(\*x^*) - f(\*x_t)} = O(B \EE[|\cT|]) = O\rbr{B \overline{T}^{(Q)}}$ and $\EE\sbr{\sum_{t \in \cT^c} f(\*x^*) - f(\*x_t)} = O\rbr{\gamma_T \sqrt{T \ln T}}$.
See Appendix~\ref{sec:proof_noisy} for the detailed proof.

Consequently, we obtain the following cumulative regret bound:
\begin{align*}
    \EE\sbr{R_T} &= 
    \begin{cases}
        O\rbr{Q \ln^{d+1} Q + \sqrt{T (\ln T)^{2d+3}}} & \text{for SE kernels}, \\
        \widetilde{O}\rbr{Q^{\tfrac{2\nu + d}{2\nu}} + T^{\tfrac{2\nu+3d}{4\nu + 2d}} } & \text{for Mat\'ern kernels}.
    \end{cases}
\end{align*}
Therefore, except for the additive factor on $Q$, our analysis achieves regret upper bounds similar to those for the sequential GP-TS \citep{Chowdhury2017-on}.
Compared with Eq.~\eqref{eq:CumualativeRegretBound_vakili}, our analysis replaces the multiplicative factor on $Q$ with the additive factor on $Q$.
This change is considered preferable in the literature \citep[see Section~4 of][]{Desautels2014-Parallelizing} since the $\gamma_T \sqrt{T \ln T}$ term is often dominant under the general assumption $T \gg Q$.

\subsection{Regret analysis for noiseless feedback}
\label{sec:NoiseFreeAnalysis}

We obtain the following lemma combining the proof by \citet{iwazaki2025gaussian} and Lemma~\ref{lem:VakiliLem3}:
\begin{lemma}[Informal]
    \label{lemPosteriorSTDUpperBoundNoiseless_informal}
    Based on the MIG bounds from \citep{vakili2021-information}, the following statements hold for any $T \in \NN$ and any input sequence $\*x_1, \ldots, \*x_T \in \cX$:
    \begin{itemize}
        \item For the SE kernels,
        \begin{align*}
            \min_{t \in [T]} \sigma_{t-Q-1}(\*x_t) &\leq 
            \sqrt{2 T / Q} \exp\rbr{-\widetilde{C}_{\rm SE} (T / Q)^{\frac{1}{d+1}}} \qquad \text{if } T\geq Q\overline{T}_{\rm SE}, \\
            \sum_{t=1}^T \min \cbr{2B, c_t \sigma_{t-Q-1}(\*x_t)} 
            &\leq 
             2BQ \overline{T}_{\rm SE}
             + \sqrt{2} c_T Q (d+1) \rbr{\frac{\widetilde{C}_{\rm SE}}{2}}^{-\frac{3d+3}{2}} \Gamma\rbr{\frac{3d + 3}{2}}.
        \end{align*}
        \item For the Mat\'ern kernels with $\nu > 1/2$,
        \begin{align*}
            \min_{t \in [T]} \sigma_{t-Q-1} (\*x_t) &\leq 
            \sqrt{2} \widetilde{C}_{\rm Mat}^{1/2} (T / Q)^{-\frac{\nu}{d}} \rbr{\ln  (T/Q)}^{\frac{\nu}{d}}  \qquad \text{if } T\geq Q\overline{T}_{\rm Mat},\\
            \sum_{t=1}^T \min \cbr{2B, c_t \sigma_{t-Q-1}(\*x_t)}
            &\leq 
            \begin{cases}
                2BQ \overline{T}_{\rm Mat}  + \sqrt{2} c_T \widetilde{C}_{\rm Mat}^{1/2} \frac{d}{d-\nu} T^{\frac{d-\nu}{d}} Q^{\frac{\nu}{d}} (\ln  (T/Q))^{\frac{\nu}{d}} & \text{if } d > \nu, \\
                2BQ \overline{T}_{\rm Mat} + \sqrt{2} c_T \widetilde{C}_{\rm Mat}^{1/2} Q (\ln  (T/Q))^{2} / 2 & \text{if } d = \nu, \\
                2BQ \overline{T}_{\rm Mat} + \sqrt{2} c_T \widetilde{C}_{\rm Mat}^{1/2} \frac{Q\Gamma(\frac{\nu}{d} + 1)}{\rbr{\frac{\nu}{d}-1}^{\frac{\nu}{d}+1}} & \text{if } d < \nu.
            \end{cases}
        \end{align*}
    \end{itemize}
    where $\widetilde{C}_{\rm SE}, \widetilde{C}_{\rm Mat}, \overline{T}_{\rm SE}$, and $\overline{T}_{\rm Mat}$ are constants with respect to $T$ and $Q$.
\end{lemma}
See Appendix~\ref{sec:proof_noiseless} for the formal version and its proof.

Hence, by combining Lemmas~\ref{lem:GeneralRegretUpperBound} and \ref{lemPosteriorSTDUpperBoundNoiseless_informal}, we derive the following result:
\begin{theorem}
    \label{thm:CumulativeRegretUpperBoundNoiseless}
    Suppose Assumptions~\ref{assump:function} and \ref{assump:discretization} hold.
    Set $1 / N_t = p_t^{(g)} = 1 / t^2$ and other variables as in Lemma~\ref{lem:GeneralRegretUpperBound}.
    Then, the following inequality holds:
    \begin{align*}
        \EE\sbr{R_T}
        &=
        \begin{cases}
            O\rbr{Q \ln^{\frac{1}{2}} T} & \text{for SE kernels}, \\
            \widetilde{O}\rbr{Q + T^{\tfrac{d-\nu}{d}} Q^{\tfrac{\nu}{d}}} & \text{for Mat\'ern kernels with $d > \nu$}, \\
            O\rbr{Q \ln^{\frac{5}{2}} T } & \text{for Mat\'ern kernels with $d = \nu$}, \\
            O\rbr{Q \ln^{\frac{1}{2}} T } & \text{for Mat\'ern kernels with $d < \nu$},
        \end{cases}
    \end{align*}
    where $\widetilde{O}$ suppresses the polylogarithmic factors for $T$ and $Q$.
\end{theorem}
As with the sequential case \citep{iwazaki2025gaussian}, we can obtain a much better regret upper bound compared with the noisy case.
In particular, for SE kernels and Mat\'ern kernels with $\nu \geq d$, we can obtain a polylogarithmic upper bound with respect to $T$.
Our regret bounds are much tighter than the known results in \citep{lyu2019efficient}.

Note that the $\ln^{\frac{1}{2}} T$ factor comes from $c_T$ and does not stem from the parallelization.
This greater dependence on $T$ is a problem of the GP-TS-based method compared with GP-UCB-based methods, as discussed in \citep{Chowdhury2017-on}.

Furthermore, modifying the proof of Lemma~\ref{lemPosteriorSTDUpperBoundNoiseless_informal}, we obtain the following simple regret bounds:
\begin{restatable}{corollary}{corSimpleRegretUpperBoundNoiseless}
    \label{cor:SimpleRegretUpperBoundNoiseless}
    Suppose Assumptions~\ref{assump:function} and \ref{assump:discretization} hold.
    Set $1 / N_t = p_t^{(g)} = 1 / t^2$  and other variables as in Lemma~\ref{lem:GeneralRegretUpperBound}.
    In addition, define $\hat{\*x}_t = \argmax_{\*x \in \{\*x_1, \dots, \*x_T \}} f(\*x)$.
    Then, the following inequality holds:
    \begin{align*}
        \EE\sbr{r_T}
        &=
        \begin{cases}
            O\rbr{Q T^{-1} \ln^{\frac{1}{2}} T} & \text{for SE kernels}, \\
            \widetilde{O}\rbr{QT^{-1} + T^{-\tfrac{\nu}{d}} Q^{\tfrac{\nu}{d}}} & \text{for Mat\'ern kernels with $d > \nu$}, \\
            O\rbr{Q T^{-1} \ln^{\frac{5}{2}} T } & \text{for Mat\'ern kernels with $d = \nu$}, \\
            O\rbr{Q T^{-1} \ln^{\frac{1}{2}} T } & \text{for Mat\'ern kernels with $d < \nu$},
        \end{cases}
    \end{align*}
    where $\widetilde{O}$ suppresses the polylogarithmic factors for $T$ and $Q$.
\end{restatable}
See Appendix~\ref{sec:proof_noiseless} for the proof.

\begin{remark}[Simple regret upper bound]
    \label{remark:simple_regret_noiseless}
    Although we obtained polynomial convergence of the simple regret for both kernels, the dependence on $T$ in these upper bounds is worse than in \citep{iwazaki2025gaussian}.
    This degradation arises from the property of GP-TS, where we cannot obtain the upper bound like $f(\*x^*) - f(\*x_t) \leq 2 \beta_t^{1/2} \sigma_{t-1}(\*x_t)$, which can be obtained by GP-UCB.
    Furthermore, even if we can obtain the simple regret bound with the order of $\min_{t \in [T]} \sigma_{t-Q-1}(\*x_t)$ in Lemma~\ref{lemPosteriorSTDUpperBoundNoiseless_informal}, the upper bounds result in the same rate as when $T / Q$ times function evaluations can be performed sequentially.
    Thus, obtaining tighter simple regret bounds for GP-TS-based methods and determining whether we can further tighten the simple regret bound for parallel GP bandit methods are important directions for future work.
\end{remark}

\section{Numerical experiments}
\label{sec:experiment}

We conducted numerical experiments for synthetic functions that satisfy the theoretical settings using scikit-learn \citep{scikit-learn}.
We set $\cX = \{0, 0.1, \dots, 0.9 \}^d$ with $d = 3$.
As in \citep{Chowdhury2017-on}, we generate sample paths from a GP, and set the posterior mean learned from data at $\cX$ as the objective function and $B^2 = \*f^\top \*K^{-1} \*f$, where $\*f = \bigl( f(\*x) \bigr)_{\*x \in \cX}$ and $\*K = \bigl( k(\*x_i, \*x_j) \bigr)_{\*x_i, \*x_j \in \cX}$.
We use the SE kernel with $\ell = 0.4$ and the Mat\'ern kernel with $\ell = 0.4$ and $\nu = 5 / 2$.
In addition, we set $\delta = 0.1$ and $\beta_t$ to its theoretical value based on the actual mutual information.
We conducted the experiments for the noiseless setting and noisy setting with Gaussian noise $\varepsilon_i \sim \cN(0, \lambda^2)$, where $\lambda = 0.1$, for all $i \in [T]$.

Figure~\ref{fig:exp} shows the mean and standard error of cumulative regrets for 16 random trials for GP sample paths generation and the randomness of the algorithm.
Throughout the experiments, the dependence on $Q$ seems to be mild, particularly for $Q \leq 4$.
These results align with our theoretical results, which imply that the dependence on $Q$ is not dominant when $T$ is sufficiently large.
Furthermore, in the noiseless setting, we can see that the cumulative regret for the SE kernel saturates around 100 iterations, which matches our Theorem~\ref{thm:CumulativeRegretUpperBoundNoiseless} that shows $O(Q + \ln^{\frac{1}{2}} T)$ cumulative regret upper bound.
For the Mat\'ern kernel, in the noiseless setting, since $d = 3 > \nu = 5/2$, the cumulative regret upper bound implied by Theorem~\ref{thm:CumulativeRegretUpperBoundNoiseless} is $\widetilde{O}\rbr{T^{\frac{1}{6}} Q^{\frac{5}{6}}}$.
Indeed, the cumulative regret exhibits sublinear growth, but its dependence on $Q$ is more moderate than the theoretical upper bound would suggest.

\begin{figure}
    \centering
    \includegraphics[width=0.49\linewidth]{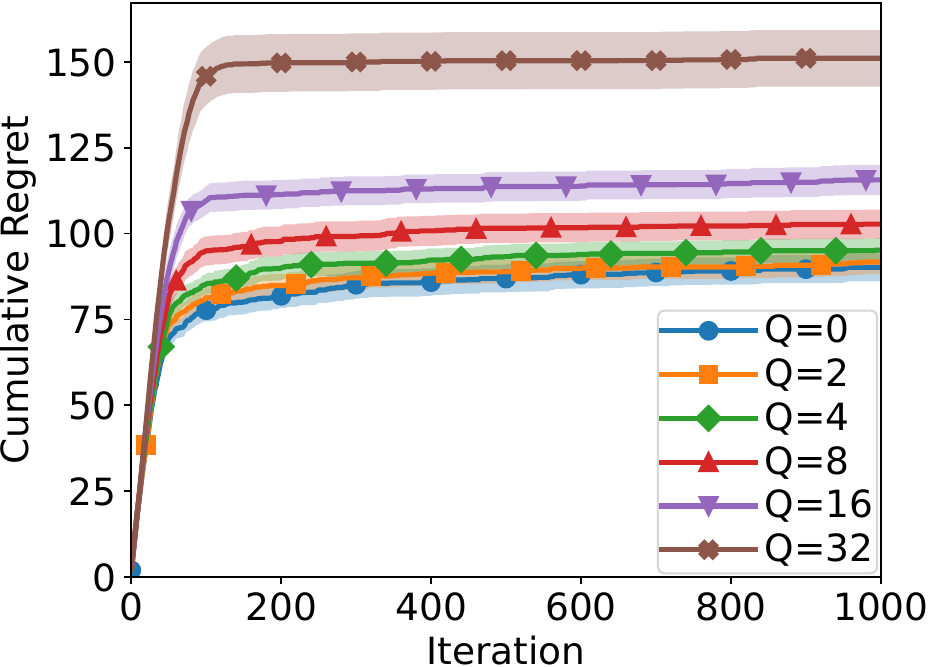}
    \includegraphics[width=0.49\linewidth]{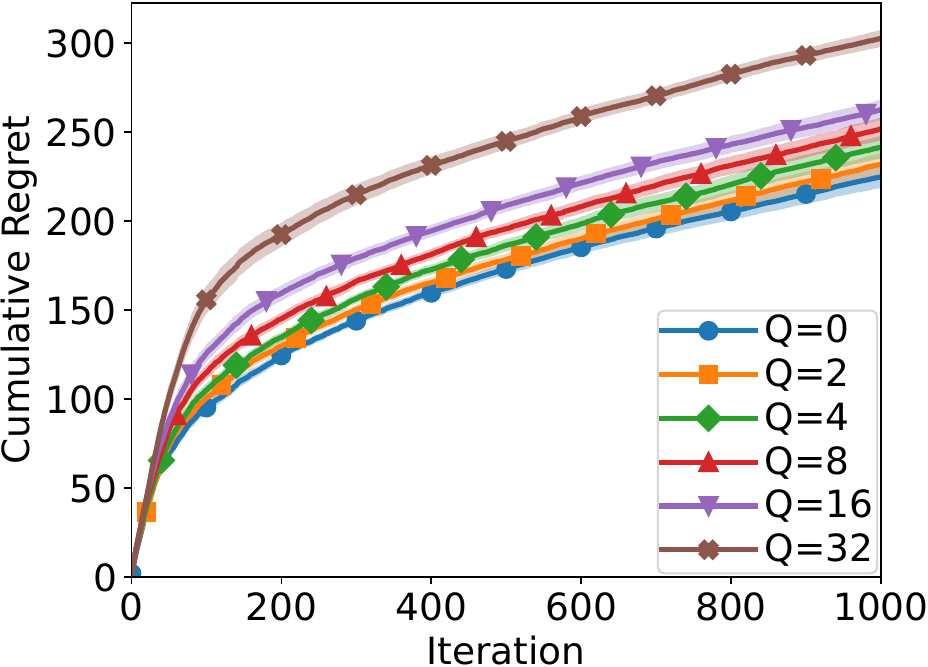}
    
    \includegraphics[width=0.49\linewidth]{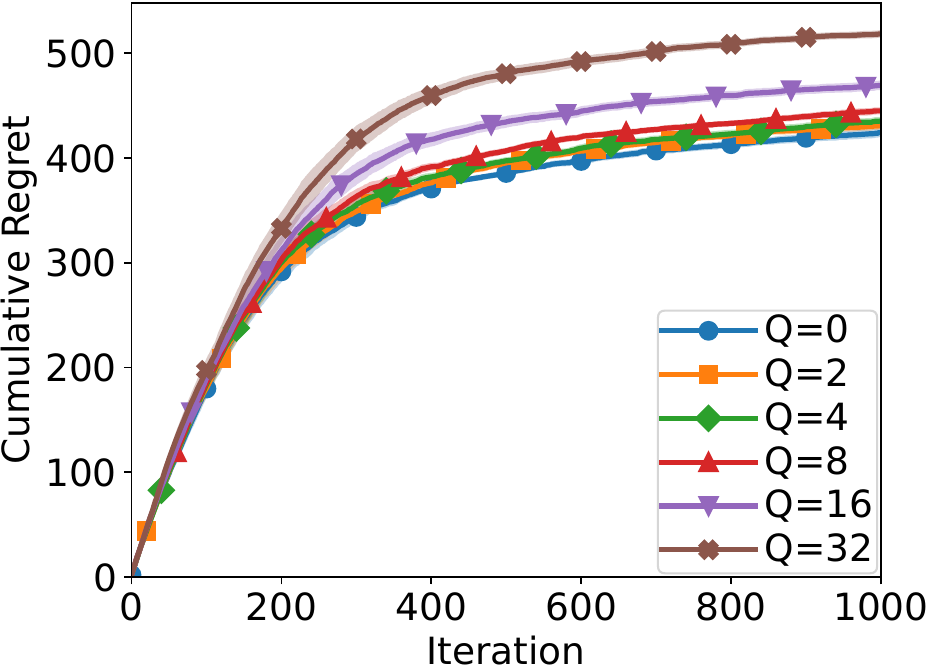}
    \includegraphics[width=0.49\linewidth]{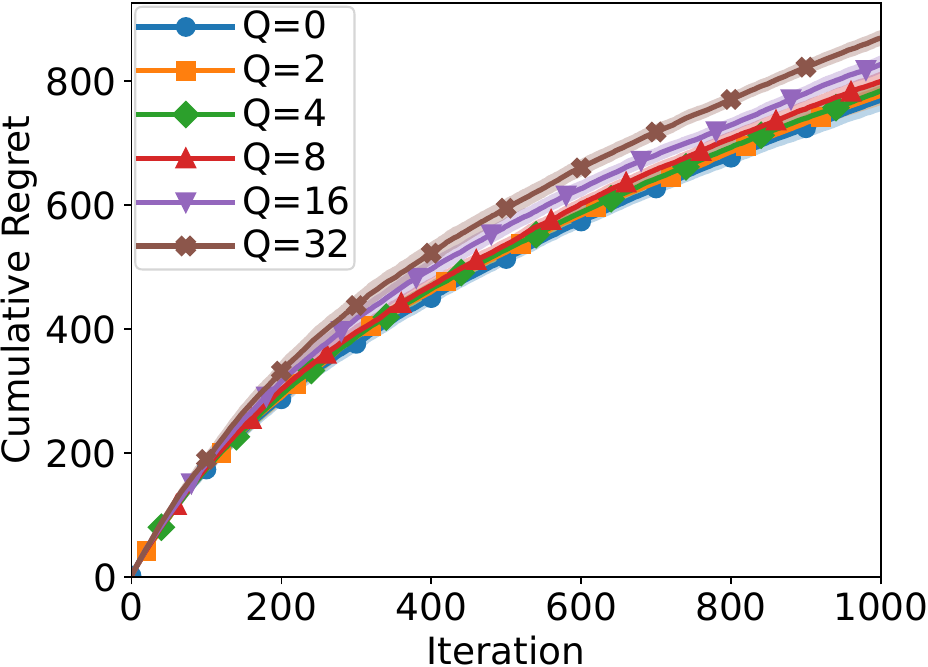}
    \caption{
    Synthetic function experiment results.
    The vertical axis shows the mean and standard error of cumulative regrets.
    The horizontal axis shows iterations.
    The left and right columns show the results in noiseless and noisy settings, respectively.
    The top and bottom rows show the results for the SE and Mat\'ern kernels, respectively.
    In each plot, we show the result of GP-BTS with $Q \in \{0, 2, 4, 8, 16, 32 \}$ shown in Algorithm~\ref{alg:GP-BTS}.
    }
    \label{fig:exp}
\end{figure}

\section{Conclusion and future work}
\label{sec:conclusion}

In this paper, we presented an improved regret analysis for GP-BTS in the simple delayed-feedback setting.
By developing Lemma~\ref{lemInflatedEllipticalPotentialLemmaRevised}, we decoupled the delay parameter $Q$ from the main regret term, converting the previously known multiplicative degradation into an additive one.
Consequently, our analysis theoretically guarantees that GP-BTS achieves cumulative regret bounds comparable to those of the sequential case without relying on the practically inefficient initial US phase \citep{Desautels2014-Parallelizing}.
Furthermore, we established regret upper bounds for the noiseless setting following \citep{iwazaki2025gaussian}.

\paragraph{Discussion and future work.}
While our results tighten the regret bounds for GP-BTS, several important avenues remain for future work:

\begin{itemize}
    \item \textbf{Simple regret bound in noiseless setting.} As discussed in Remark~\ref{remark:simple_regret_noiseless}, determining the possibility of improving our simple regret bounds in the noiseless setting is crucial.
    \item \textbf{Adaptation to GP-BUCB.} Our current proof technique cannot directly tighten the regret bounds of GP-BUCB, since it uses the confidence parameter $\beta_T = O(Q \gamma_T)$, which is already inflated by $Q$ \citep{Desautels2014-Parallelizing}.
    Furthermore, unlike GP-BTS, GP-BUCB may not select diverse points in a batch if we use $\beta_t = \Theta(\gamma_t)$, which is not scaled by $Q$, although we conjecture that such an algorithm can achieve the same regret upper bound as ours.
    Thus, to develop a meaningful UCB-based batched method achieving the same regret upper bound as ours, we need to design a new UCB-based method that can select diverse points in a batch without scaling $\beta_t$ by $Q$.
    \item \textbf{High-probability bounds.} Since we present expected regret bounds, Markov's inequality readily indicates high-probability bounds with $1/\delta$ dependence.
    For the noisy case, obtaining a $\ln(1/\delta)$ dependence is relatively straightforward using the Azuma-Hoeffding inequality as with \citep{Chowdhury2017-on}.
    However, deriving a $\ln(1/\delta)$ dependence in the noiseless setting remains an open problem.
    This is because the standard concentration inequalities incur an $O(\sqrt{T})$ dependence, which is dominant in the noiseless case.
    \item \textbf{Bayesian setting adaptation.} 
    Tightening the regret bounds in the Bayesian setting \citep{Kandasamy2018-Parallelised,nava2022diversified,sugiura2026randomized} is a promising direction.
    \item \textbf{Discretization error for GP-TS.} 
    Our analysis for GP-BTS incurs discretization error as with GP-TS \citep{Chowdhury2017-on}, which results in $\ln T$ degradation in the noisy setting and can be non-negligible in the noiseless setting.
    Carefully addressing the need for discretization to avoid this degradation is an important future direction.
    \item \textbf{Optimality with respect to $Q$.}
    Although we have reduced the dependence on $Q$, the optimal dependence remains unknown.
    The additive $\Theta(Q)$ term is unavoidable, since the first $Q$ evaluations can incur $\Theta(Q)$ regret in the worst case.
    On the other hand, it may be possible to eliminate the additional factors, such as $O(\ln^{d+1} Q)$ and $O(Q^{d/(2\nu)})$ for the SE and Mat\'ern kernels, respectively, in the noisy setting, and the multiplicative $Q^{\nu/d}$ factor for the Mat\'ern kernel with $d > \nu$ in the noiseless setting.
    We leave a characterization of the optimal dependence on $Q$ for future work.
\end{itemize}

\bibliography{ref}
\bibliographystyle{plainnat}
\appendix
\section{Auxiliary lemmas}
\label{sec:lemmas}

This section lists the auxiliary lemmas for the proof.

We use the following lemma to obtain the upper bound of the discretization error.
\begin{lemma}[Lemma~5.1 in \citep{freitas2012exponential}]
    Suppose that Assumption~\ref{assump:discretization} holds.
    Then, any $g \in \cH_k$ is Lipschitz continuous with respect to $\|g\|_{\cH_k} L_k$.
    \label{lem:RKHS_lipschitz}
\end{lemma}

The following lemmas provide the confidence intervals for both noisy and noiseless settings:
\begin{lemma}[Proposition 1 in \citep{vakili2021-optimal}, Corollary 3.11 in \citep{kanagawa2018gaussian}, Corollary 5.6 in \citep{kanagawa2025gaussian}, or Lemma 6 in \citep{freitas2012exponential}]
    Suppose that Assumption~\ref{assump:function} holds and $f(\*x_t)$ can be obtained for all $t \in \NN$.
    Then, by setting $\lambda^2 \geq 0$, the following holds:
    \begin{align}
        \forall t \in \NN, \forall \*x \in \cX, |f(\*x) - \mu_{t-1}(\*x)| \leq B \sigma_{t-1}(\*x),
    \end{align}
    for any input sequence $(\*x_t)_{t \in \NN}$.
    \label{lem:CI_noiseless}
\end{lemma}
\begin{lemma}[Adapted from Theorem 3.11 in \citep{abbasi2013online}]
    Suppose that Assumptions~\ref{assump:function} and \ref{assump:noisy} hold.
    Fix $\delta \in (0, 1)$.
    Then, by setting $\lambda^2 > 0$, the following holds:
    \begin{align}
        \Pr\left(\forall t \in \NN, \forall \*x \in \cX, |f(\*x) - \mu_{t-1}(\*x_t)| \leq \beta_t^{1/2} \sigma_{t-1}(\*x_t) \right) \geq 1 - \delta,
    \end{align}
    where $\beta_t^{1/2} = B + \frac{R}{\lambda} \sqrt{2\gamma_t + 2\ln (1 / \delta)}$.
    \label{lem:CI_noisy}
\end{lemma}

We use the following useful lemmas for noisy settings:
%
%
\begin{restatable}[Extended from Proposition 1 of \citet{Desautels2014-Parallelizing}]{lemma}{lemPoteriorVarianceInequality}
    \label{lemPoteriorVarianceInequality}
    Fix $\lambda > 0$, any dataset $\cD_{(1)}$, and $\cD_{(2)} = \{ (\*x_{(2), i}, y_{(2), i}) \}_{i=1}^N$ for some $N \in \NN_+$.
    Define 
    $\sigma_{(1)}^2 (\*x)$ and $\sigma_{(1 \cup 2)}^2 (\*x)$ as the posterior variance functions computed from $\cD_{(1)}$ and $\cD_{(1)} \cup \cD_{(2)}$, respectively.
    Furthermore, define $k_{(1)}(\*x, \*x^\prime)$ as the covariance function computed from $\cD_{(1)}$.
    Then, the following holds:
    \begin{align}
        \frac{\sigma_{(1)}^2(\*x) }{\sigma_{(1 \cup 2)}^2(\*x)}
        &= 
        \frac{{\rm det}\rbr{\*\Sigma_{(1)}(\*X_{(2)}) + \lambda^2 \*I_N}}{{\rm det}\rbr{\*R + \lambda^2 \*I_N} }
        \leq 1 + \lambda^{-2} \sum_{i=1}^N \sigma_{(1)}^2 (\*x_{(2), i}),
    \end{align}
    where
    \begin{itemize}
        \item $\*\Sigma_{(1)}(\*X_{(2)}) \in \RR^{N \times N}$ denotes the posterior covariance matrix whose $(i, j)$-th element is $k_{(1)}(\*x_{(2), i}, \*x_{(2), j})$,
        \item $\*k_{(1)}(\*x, \*X_{(2)}) \in \RR^N$ denotes the posterior covariance vector whose $i$-th element is $k_{(1)}(\*x, \*x_{(2), i})$,
    \end{itemize}
    and 
    \begin{align*}
        \*R = \*\Sigma_{(1)}(\*X_{(2)}) - \frac{\*k_{(1)}(\*x, \*X_{(2)}) \*k_{(1)}(\*x, \*X_{(2)})^\top }{\sigma_{(1)}^2 (\*x)}.
    \end{align*}
\end{restatable}
\begin{proof}
    By adapting Proposition~1 of \citet{Desautels2014-Parallelizing} to our simple delayed setting, we can derive the following equation:
    \begin{align}
        \frac{\sigma_{(1)}^2 (\*x) }{\sigma_{(1 \cup 2)}^2(\*x)}
        &= 
        \exp\cbr{
        2 I\rbr{f(\*x) ; \*y_{(2)} \mid \cD_{(1)}}
        },
    \end{align}
    where
    $\*y_{(2)} = \rbr{y_{(2), 1}, \dots, y_{(2), N}} \in \RR^N$
    and $I(\*Z_1; \*Z_2 \mid \cD)$ is the conditional Shannon mutual information between $\*Z_1$ and $\*Z_2$ given $\cD$.
    From the definition of the mutual information of a multivariate Gaussian distribution, we have
    \begin{align}
        2 I\rbr{f(\*x) ; \*y_{(2)} \mid \cD_{(1)}}
        &= 2 \sbr{H(\*y_{(2)} \mid \cD_{(1)}) - H(\*y_{(2)} \mid \cD_{(1)} \cup \{ \rbr{\*x, f(\*x)} \})} \\
        &= \ln\rbr{ \rbr{2 \pi e}^N {\rm det}\rbr{\*\Sigma_{(1)}(\*X_{(2)}) + \lambda^2 \*I_N} } \\
        &\ - \ln\rbr{ \rbr{2 \pi e}^N {\rm det}\rbr{\*\Sigma_{(1)}(\*X_{(2)}) + \lambda^2 \*I_N - \frac{\*k_{(1)}(\*x, \*X_{(2)}) \*k_{(1)}(\*x, \*X_{(2)})^\top }{\sigma_{(1)}^2 (\*x)}} } \\
        &= \ln \frac{{\rm det}\rbr{\*\Sigma_{(1)}(\*X_{(2)}) + \lambda^2 \*I_N}}{{\rm det}\rbr{\*R + \lambda^2 \*I_N} }.
    \end{align}
    Hence, we obtain
    \begin{align}
        \frac{\sigma_{(1)}^2 (\*x) }{\sigma_{(1 \cup 2)}^2(\*x)}
        &= 
        \exp\cbr{\rbr{ 
        \ln \frac{{\rm det}\rbr{\*\Sigma_{(1)}(\*X_{(2)}) + \lambda^2 \*I_N}}{{\rm det}\rbr{\*R + \lambda^2 \*I_N} }
        }} \\
        &= \frac{{\rm det}\rbr{\*\Sigma_{(1)}(\*X_{(2)}) + \lambda^2 \*I_N}}{{\rm det}\rbr{\*R + \lambda^2 \*I_N} }.
    \end{align}

    Furthermore, by using the matrix determinant lemma ${\rm det}(\*A + \*a \*b^\top) = {\rm det}(\*A) \rbr{1 + \*b^\top \*A^{-1} \*a}$ for any positive definite matrix $\*A$ and any vectors $\*a$ and $\*b$ as $\*A = \*R + \lambda^2 \*I_N$, $\*a = \*k_{(1)}(\*x, \*X_{(2)})$, and $\*b = \frac{\*k_{(1)}(\*x, \*X_{(2)})}{\sigma_{(1)}^2 (\*x)}$, we have
    \begin{align}
        \frac{\sigma_{(1)}^2 (\*x) }{\sigma_{(1 \cup 2)}^2(\*x)}
        &= 1 + \frac{\*k_{(1)}(\*x, \*X_{(2)})^\top}{\sigma_{(1)}^2 (\*x)} \rbr{\*R + \lambda^2 \*I_N}^{-1} \*k_{(1)}(\*x, \*X_{(2)}).
    \end{align}
    Since $\*R$ is positive semidefinite, we see 
    $\*a^\top \rbr{\lambda^2 \*I_N}^{-1} \*a \geq \*a^\top \rbr{\*R + \lambda^2 \*I_N}^{-1} \*a$ for any vector $\*a$.
    Therefore, we can obtain
    \begin{align}
        \frac{\sigma_{(1)}^2 (\*x) }{\sigma_{(1 \cup 2)}^2(\*x)}
        &\leq 1 + \frac{\sum_{i=1}^N k_{(1)}^2(\*x, \*x_{(2), i})}{\lambda^2 \sigma_{(1)}^2 (\*x)}.
    \end{align}
    Since $k_{(1)}^2(\*x, \*x_{(2), i}) \leq \sigma_{(1)}^2 (\*x) \sigma_{(1)}^2 (\*x_{(2), i})$, we further obtain
    \begin{align}
        \frac{\sigma_{(1)}^2 (\*x) }{\sigma_{(1 \cup 2)}^2(\*x)}
        &\leq 1 + \lambda^{-2} \sum_{i=1}^N \sigma_{(1)}^2 (\*x_{(2), i}).
    \end{align}
\end{proof}


\section{Proof}
\label{sec:proof}

\subsection{Proof for Lemma~\ref{lem:GeneralRegretUpperBound}}
\label{sec:proof_GeneralRegretUpperBoundNoisy}

First, we show the following lemma, modified from Lemma 10 in \citep{Chowdhury2017-on}:
\begin{lemma}
    \label{lem:GP-TS-InstanteneousRegret}
    Assume the same premise as in Theorem~\ref{TheoCumulativeRegretBoundNoisy}.
    Furthermore, define the event $E^{(f)}_t$ as follows:
    \begin{align}
        E^{(f)}_t = \{\forall \*x \in \cX_t, |f(\*x) - \mu_{t-Q-1}(\*x)| \leq \beta_t^{1/2}\bigl(p_t^{(f)}\bigr) \sigma_{t-Q-1}(\*x) \}
    \end{align}
    where $\cX_t \subset \cX$ is a finite subset of an input domain.
    Then, for all $t \in [T]$ and $\cD_{t-Q-1}$ such that $E^{(f)}_t$ is true, the following inequality holds:
    \begin{align}
        \EE\left[ f([\*x^*]_t) - f(\*x_t) \mid \cD_{t-Q-1} \right]
        &\leq 
        \frac{2 + p}{p} \EE\left[ \min \left\{2B, c_t \sigma_{t-Q-1} (\*x_t)\right\} \mid \cD_{t-Q-1} \right] + 2B p_t^{(g)},
    \end{align}
    where 
    $[\*x]_t = \argmin_{\*x^\prime \in \cX_t} \|\*x - \*x^\prime \|_1$,
    $p = 1 / (4 e \sqrt{\pi})$,
    and $c_t = \beta_t^{1/2}\bigl(p_t^{(f)}\bigr) \left(2 \zeta_t^{1/2}\bigl(p_t^{(g)}\bigr) + 1 \right)$.
\end{lemma}
\begin{proof}
    First, we show that, for any $\cD_{t-Q-1}$ such that $E^{(f)}_t$ is true, 
    \begin{align}
        \Pr\left( \*x_t \in \cX_t \backslash \cS_t \mid \cD_{t-Q-1} \right) \geq \frac{p}{2},
    \end{align}
    where $\cS_t$ is called saturation set \citep{Chowdhury2017-on} and is defined as follows:
    \begin{align}
        \cS_t = \cbr{\*x \in \cX_t \biggm| f([\*x^*]_t) - f(\*x) > 
        \beta_t^{1/2}\bigl(p_t^{(f)}\bigr) \left(\zeta_t^{1/2}\bigl(p_t^{(g)}\bigr) + 1 \right) \sigma_{t-Q-1} (\*x) }.
    \end{align}

    To show the above inequality, we use the following facts.
    For any $\cD_{t-Q-1}$ such that $E^{(f)}_t$ is true and $\sigma_{t-Q-1} ([\*x^*]_t) = 0$, the probability $\Pr\rbr{g_t([\*x^*]_t) = f([\*x^*]_t) = \mu_{t-Q-1}([\*x^*]_t) \mid \cD_{t-Q-1}} = 1$.
    Furthermore, the following inequality holds for any $\cD_{t-Q-1}$ such that $E^{(f)}_t$ is true and $\sigma_{t-Q-1} ([\*x^*]_t) > 0$:
    \begin{align}
        &\Pr\left( g_t([\*x^*]_t) \geq f([\*x^*]_t) \mid \cD_{t-Q-1} \right) \\
        &\geq \Pr\left( g_t([\*x^*]_t) \geq \mu_{t - Q - 1}([\*x^*]_t) + \beta_t^{1/2}\bigl(p_t^{(f)}\bigr) \sigma_{t-Q-1}([\*x^*]_t) \mid \cD_{t-Q-1} \right) \\
        &= \Pr\left( \frac{g_t([\*x^*]_t) - \mu_{t - Q - 1}([\*x^*]_t)}{\beta_t^{1/2}\bigl(p_t^{(f)}\bigr) \sigma_{t-Q-1} ([\*x^*]_t)} \geq 1 \biggm| \cD_{t-Q-1} \right) \\
        &\geq p,
    \end{align}
    where the first inequality holds since $E^{(f)}_t$ is true and the second inequality holds due to $\frac{g_t([\*x^*]_t) - \mu_{t - Q - 1}([\*x^*]_t)}{\beta_t^{1/2}\bigl(p_t^{(f)}\bigr) \sigma_{t-Q-1} ([\*x^*]_t)} \mid \cD_{t-Q-1} \sim \cN(0, 1)$ and Gaussian anti-concentration inequality (e.g., Lemma 7 in \citep{Chowdhury2017-on} and and Lemma 5.1 in \citep{Srinivas2010-Gaussian}).
    Hence, $\Pr\rbr{g_t([\*x^*]_t) \geq f([\*x^*]_t) \mid \cD_{t-Q-1}} \geq p$ for any $\cD_{t-Q-1}$ such that $E^{(f)}_t$ is true.
    Furthermore, define the event $E^{(g)}_t$ as follows:
    \begin{align}
        E^{(g)}_t = \{\forall \*x \in \cX_t, |g_t(\*x) - \mu_{t-Q-1}(\*x)| 
        \leq \beta_t^{1/2}\bigl(p_t^{(f)}\bigr) \zeta_t^{1/2}\bigl(p_t^{(g)}\bigr) \sigma_{t-Q-1}(\*x) \}.
    \end{align}
    From Gaussian concentration inequality (e.g., Lemma 5 in \citep{Chowdhury2017-on} and and Lemma 5.1 in \citep{Srinivas2010-Gaussian}) and the definition of $\zeta_t (\delta) = 2 \ln\bigl( 2 (L_k drN_t)^d / \bigl(p \delta \bigr) \bigr)$, we have $\Pr(E^{(g)}_t \mid \cD_{t-Q-1}) \geq 1 - \frac{p}{2} p_t^{(g)}$ for any $\cD_{t-Q-1}$.
    Therefore, for any $\cD_{t-Q-1}$ such that $E^{(f)}_t$ is true,
    \begin{align}
        &\Pr\left(\exists \*x \in \cS_t, f([\*x^*]_t) < g_t(\*x) \mid \cD_{t-Q-1} \right) \\
        &\leq 
        \Pr\left(\exists \*x \in \cS_t, 
        f(\*x) + \beta_t^{1/2}\bigl(p_t^{(f)}\bigr) \left(\zeta_t^{1/2}\bigl(p_t^{(g)}\bigr) + 1 \right)
        \sigma_{t-Q-1} (\*x) < g_t(\*x) \mid \cD_{t-Q-1} \right) \\
        &\leq \frac{p}{2} p_t^{(g)},
    \end{align}
    which holds due to the definition of $\cS_t$ and the probability bound for $E^{(g)}_t$.

    Obviously, $\cX_t \backslash \cS_t$ is not empty since $[\*x^*]_t$ belongs to $\cX_t \backslash \cS_t$.
    If $g_t([\*x^*]_t) \geq g_t(\*x)$ for all $\*x \in \cS_t$, then $\*x_t \in \cX_t \backslash \cS_t$.
    Thus, for any $\cD_{t-Q-1}$ such that $E^{(f)}_t$ is true, we can obtain the lower bound as follows:
    \begin{align}
        &\Pr\left( \*x_t \in \cX_t \backslash \cS_t \mid \cD_{t-Q-1} \right) \\
        &\geq \Pr\left(\forall \*x \in \cS_t, g_t([\*x^*]_t) \geq g_t(\*x) \mid \cD_{t-Q-1} \right) \\
        &\geq \Pr\left(\forall \*x \in \cS_t, g_t([\*x^*]_t) \geq f([\*x^*]_t) \geq g_t(\*x) \mid \cD_{t-Q-1} \right) \\
        &\geq \Pr\left( g_t([\*x^*]_t) \geq f([\*x^*]_t) \mid \cD_{t-Q-1} \right) - \Pr\left(\exists \*x \in \cS_t, f([\*x^*]_t) < g_t(\*x) \mid \cD_{t-Q-1} \right) \\
        &\geq p - \frac{p}{2} p_t^{(g)} \geq \frac{p}{2}.
    \end{align}

    Define $\bar{\*x}_t = \argmin_{\*x \in \cX_t \backslash \cS_t} \sigma_{t - Q - 1}(\*x)$.
    Then, since $\bar{\*x}_t$ is deterministic given $\cD_{t-Q-1}$, we can see that, for any $\cD_{t-Q-1}$ such that $E^{(f)}_t$ is true,
    \begin{align}
        &\EE\left[ \min \left\{2B, c_t \sigma_{t-Q-1}(\*x_t) \right\} \mid \cD_{t-Q-1} \right] \\
        &\geq \EE\left[ \min \left\{2B, c_t \sigma_{t-Q-1}(\*x_t) \right\} \mid \cD_{t-Q-1}, \*x_t \in \cX_t \backslash \cS_t \right] \Pr\left( \*x_t \in \cX_t \backslash \cS_t \mid \cD_{t-Q-1} \right) \\
        &\geq \frac{p\min \left\{2B, c_t \sigma_{t-Q-1}(\bar{\*x}_t) \right\}}{2},
        \label{eq:deterministic_sigma_lower_bound}
    \end{align}
    where $c_t = \beta_t^{1/2}\bigl(p_t^{(f)}\bigr) \left(2 \zeta_t^{1/2}\bigl(p_t^{(g)}\bigr) + 1 \right)$.

    If $E^{(f)}_t$ and $E^{(g)}_t$ are true, then
    \begin{align}
        f([\*x^*]_t) - f(\*x_t)
        &= 
        f([\*x^*]_t) - f(\bar{\*x}_t)
        + f(\bar{\*x}_t) - g_t(\bar{\*x}_t)
        + g_t(\bar{\*x}_t) - g_t(\*x_t)
        + g_t(\*x_t) - f(\*x_t)
        \\
        &\leq 
        \beta_t^{1/2}\bigl(p_t^{(f)}\bigr) \zeta_t^{1/2}\bigl(p_t^{(g)}\bigr) \sigma_{t-Q-1} (\bar{\*x}_t)
        + f(\bar{\*x}_t) - g_t(\bar{\*x}_t) + g_t(\*x_t) - f(\*x_t) \\
        &\leq 
        c_t \sigma_{t-Q-1} (\bar{\*x}_t) + \beta_t^{1/2}\bigl(p_t^{(f)}\bigr) \left(\zeta_t^{1/2}\bigl(p_t^{(g)}\bigr) + 1 \right) \sigma_{t-Q-1} (\*x_t) \\
        &\leq 
        c_t \left(\sigma_{t-Q-1} (\bar{\*x}_t) + \sigma_{t-Q-1} (\*x_t)\right),
    \end{align}
    where the first inequality follows from $\bar{\*x}_t \in \cX_t \backslash \cS_t$ and $g_t(\bar{\*x}_t) - g_t(\*x_t) \leq 0$, and the second inequality holds since $E^{(f)}_t$ and $E^{(g)}_t$ are true.
    Furthermore, since $f([\*x^*]_t) - f(\*x_t) \leq 2B$ from 
    Assumption~\ref{assump:function}, we can see that
    \begin{align}
        f([\*x^*]_t) - f(\*x_t)
        &\leq \min\{2B, c_t \left(\sigma_{t-Q-1} (\bar{\*x}_t) + \sigma_{t-Q-1} (\*x_t)\right) \}\\
        &\leq \min \left\{2B, c_t \sigma_{t-Q-1} (\bar{\*x}_t) \right\} + \min \left\{2B, c_t \sigma_{t-Q-1} (\*x_t) \right\},
        \label{eq:deterministic_instanteneous_bound}
    \end{align}
    if $E^{(f)}_t$ and $E^{(g)}_t$ are true.
    This is because, the last upper bound $\min \left\{2B, c_t \sigma_{t-Q-1} (\bar{\*x}_t) \right\} + \min \left\{2B, c_t \sigma_{t-Q-1} (\*x_t) \right\}$ is $4B$, $2B + c_t \sigma_{t-Q-1} (\bar{\*x}_t)$, $2B + c_t \sigma_{t-Q-1} (\*x_t)$, or $c_t \left(\sigma_{t-Q-1} (\bar{\*x}_t) + \sigma_{t-Q-1} (\*x_t) \right)$.
    Hence, if $2B \leq c_t \left(\sigma_{t-Q-1} (\bar{\*x}_t) + \sigma_{t-Q-1} (\*x_t) \right)$, then $2B$ is less than the above four quantities, and vice versa.
    Thus, for any $\cD_{t-Q-1}$ such that $E^{(f)}_t$ is true, we can obtain the following bound:
    \begin{align}
        &\EE\left[ f([\*x^*]_t) - f(\*x_t) \mid \cD_{t-Q-1} \right] \\
        &\leq 
        \EE\left[ \min \left\{2B, c_t \sigma_{t-Q-1} (\bar{\*x}_t) \right\} \mid E^{(g)}_t, \cD_{t-Q-1} \right] \Pr(E^{(g)}_t \mid \cD_{t-Q-1}) \\
        &\qquad + \EE\left[ \min \left\{2B, c_t \sigma_{t-Q-1} (\*x_t) \right\} \mid E^{(g)}_t, \cD_{t-Q-1} \right] \Pr(E^{(g)}_t \mid \cD_{t-Q-1}) \\
        &\qquad + \EE\left[ f([\*x^*]_t) - f(\*x_t) \mid \overline{E^{(g)}_t}, \cD_{t-Q-1} \right] \Pr(\overline{E^{(g)}_t} \mid \cD_{t-Q-1}) \\
        &\leq 
        \EE\left[ \min \left\{2B, c_t \sigma_{t-Q-1} (\bar{\*x}_t)\right\} \mid \cD_{t-Q-1} \right] \\
        &\qquad + \EE\left[ \min \left\{2B, c_t \sigma_{t-Q-1} (\*x_t)\right\} \mid \cD_{t-Q-1} \right] + 2B p_t^{(g)} \\
        &\leq 
        \frac{2 + p}{p} \EE\left[ \min \left\{2B, c_t \sigma_{t-Q-1} (\*x_t)\right\} \mid \cD_{t-Q-1} \right] + 2B p_t^{(g)},
    \end{align}
    where $\overline{E^{(g)}_t}$ is a complementary event of $E^{(g)}_t$ and we use Eqs.~\eqref{eq:deterministic_sigma_lower_bound} and \eqref{eq:deterministic_instanteneous_bound}, $\max_{\*x \in \cX} |f(\*x)| \leq B$ from Assumption~\ref{assump:function}.
    %
\end{proof}

\lemGeneralRegretUpperBound*
\begin{proof}
    First, from Assumption~\ref{assump:discretization}, Lemma~\ref{lem:RKHS_lipschitz} and $\| f \|_{\cH_k} \leq B$ (Assumption~\ref{assump:function}), we can see that
    \begin{align}
        \max_{\*x \in \cX} |f(\*x) - f([\*x]_t)|
        &\leq BL_k \max_{\*x \in \cX} \| \*x - [\*x]_t \|_1
        \leq \frac{B}{N_t},
    \end{align}
    where $[\*x]_t = \argmin_{\*x^\prime \in \cX_t} \|\*x - \*x^\prime \|_1$.
    Furthermore, define the event $E^{(f)}_t$ as follows:
    \begin{align}
        E^{(f)}_t = \{\forall \*x \in \cX, |f(\*x) - \mu_{t-1}(\*x)| \leq \beta_t^{1/2}\bigl(p_t^{(f)}\bigr) \sigma_{t-1}(\*x) \}.
    \end{align}
    For the noisy setting, from Lemma~\ref{lem:CI_noisy}, $\Pr(E^{(f)}_t) \geq 1 - p_t^{(f)}$ and $\Pr\bigr(\overline{E^{(f)}_t} \bigr) \leq p_t^{(f)}$, where $\overline{E^{(f)}_t}$ is a complementary event of $E^{(f)}_t$.
    For the noiseless setting, from Lemma~\ref{lem:CI_noiseless}, $\Pr(E^{(f)}_t) = 1$ and $\Pr\bigr(\overline{E^{(f)}_t} \bigr) = 0$.
    Therefore, we can arrange the instantaneous regret as follows:
    \begin{align}
        &\EE\left[ f(\*x^*) - f(\*x_t) \right] \\
        &\leq 
        |f(\*x^*) - f([\*x^*]_t)| +
        \EE\left[ f([\*x^*]_t) - f(\*x_t) \right] \\
        &\leq 
        \frac{B}{N_t} 
        + \EE\left[ (f([\*x^*]_t) - f(\*x_t)) \mathbbm{1}\rbr{E^{(f)}_{t-Q}} \right] 
        + \EE\left[ (f([\*x^*]_t) - f(\*x_t)) \mathbbm{1}\rbr{\overline{E^{(f)}_{t-Q}}} \right] \\
        &\leq 
        \frac{B}{N_t} 
        + \EE\left[ (f([\*x^*]_t) - f(\*x_t)) \mathbbm{1}\rbr{E^{(f)}_{t-Q}} \right] 
        + 2B p_t^{(f)} \\
        &=
        \frac{B}{N_t} 
        + \EE_{\cD_{t-Q-1}} \sbr{\EE\bigl[ f([\*x^*]_t) - f(\*x_t) \mathbbm{1}\rbr{E^{(f)}_{t-Q}} \mid \cD_{t-Q-1} \bigr]} 
        + 2B p_t^{(f)}.
    \end{align}
    where we use 
    (i) $f$ and $\cX_t$ do not depend on the random quantities for the first line, 
    (ii) the property of the discretization and the decomposition of the expectation for the second line, 
    (iii) $\max_{\*x \in \cX} |f(\*x)| \leq B$ from Assumption~\ref{assump:function} and $\Pr(\overline{E^{(f)}_t}) \leq p_t^{(f)}$ for the third line, and
    (iv) tower property of expectation for the fourth line.

    Thus, by combining Lemma~\ref{lem:GP-TS-InstanteneousRegret}, we can see that
    \begin{align}
        &\EE\left[ f(\*x^*) - f(\*x_t) \right] \\
        &\leq 
        \frac{2 + p}{p} 
        \EE\sbr{ \min \left\{2B, c_t \sigma_{t-Q-1} (\*x_t)\right\} \mathbbm{1}\rbr{E^{(f)}_{t-Q}} }
        + \frac{B}{N_t} 
        + 2B p_t^{(f)}
        + 2B p_t^{(g)} \\
        &\leq 
        \frac{2 + p}{p} \EE\left[ \min \left\{2B, c_t \sigma_{t-Q-1} (\*x_t)\right\} \right]
        + \frac{B}{N_t} 
        + 2B p_t^{(f)}
        + 2B p_t^{(g)},
    \end{align}
    where $c_t = \beta_t^{1/2}\bigl(p_t^{(f)}\bigr) \left(2 \zeta_t^{1/2}\bigl(p_t^{(g)}\bigr) + 1 \right) $.
\end{proof}

\subsection{Proof for Lemma~\ref{lemInflatedEllipticalPotentialLemmaRevised}}
\label{sec:proof_lemInflatedEllipticalPotentialLemma}

\lemInflatedEllipticalPotentialLemmaRevised*
\begin{proof}
    First, we show $|\cT| \leq 6 Q \gamma_{|\cT|} (\lambda^2) + Q \leq 6 Q \gamma_{T} (\lambda^2) + Q$.
    Let $\widetilde{\cT} = \{t \in [T] \mid \sigma_{t - Q - 1}^2(\*x_t) \geq \lambda^2 / Q , t > Q\}$ to handle the initial rounds separately.
    As with the proof for the elliptical potential count lemma (Lemma D.9 in \citet{flynn2025-tighter}, Lemma 3.3 in \citet{iwazaki2025-improvedGPbandit}, Lemma 6 in \citet{iwazaki2025gaussian}), we can see that
    \begin{align}
        |\widetilde{\cT}|
        &= \sum_{t \in \widetilde{\cT}} h\rbr{ Q \lambda^{-2} \sigma_{t-Q-1}^2 (\*x_t) },
    \end{align}
    where we denote $h(\cdot) = \min\{1, \cdot \}$ for simplicity.
    Furthermore, by telescoping, we obtain
    \begin{align}
        |\widetilde{\cT}|
        &= \sum_{t \in \widetilde{\cT}} h\rbr{ Q \lambda^{-2} \sigma_{t-1}^2 (\*x_t) }
        + \sum_{t \in \widetilde{\cT}} \sum_{i=1}^Q \cbr{ h\rbr{ Q \lambda^{-2} \sigma_{t-i-1}^2 (\*x_t) } - h\rbr{ Q \lambda^{-2} \sigma_{t-i}^2 (\*x_t) }}.
        \label{eq:count_telescoping}
    \end{align}
    From the definition of the posterior variance function, we see
    \begin{align}
        \sigma_{t-i}^2(\*x_t) 
        &= \sigma_{t-i-1}^2(\*x_t) - \frac{k_{t-i-1}^2(\*x_{t-i}, \*x_t) }{\sigma_{t-i-1}^2(\*x_{t-i}) + \lambda^2} \\
        &\geq \sigma_{t-i-1}^2(\*x_t) - \frac{\sigma_{t-i-1}^2(\*x_{t-i}) \sigma_{t-i-1}^2(\*x_t) }{\sigma_{t-i-1}^2(\*x_{t-i}) + \lambda^2} \\
        &= \frac{\sigma_{t-i-1}^2(\*x_t) }{\lambda^{-2} \sigma_{t-i-1}^2(\*x_{t-i}) + 1}.
        \label{eq:post_var_onestep_LB}
    \end{align}
    Hence, for any $i \in [Q]$ and $t \in \widetilde{\cT}$, we have
    \begin{align}
        &h\rbr{ Q \lambda^{-2} \sigma_{t-i-1}^2 (\*x_t) } - h\rbr{ Q \lambda^{-2} \sigma_{t-i}^2 (\*x_t)} \\
        &\overset{(i)}{\leq} h\rbr{ Q \lambda^{-2} \sigma_{t-i-1}^2 (\*x_t) } - h\rbr{ Q \lambda^{-2} \frac{\sigma_{t-i-1}^2(\*x_t) }{\lambda^{-2} \sigma_{t-i-1}^2(\*x_{t-i}) + 1}} \\
        &\overset{(ii)}{\leq} \rbr{1 - \frac{1}{\lambda^{-2} \sigma_{t-i-1}^2(\*x_{t-i}) + 1}} h\rbr{ Q \lambda^{-2} \sigma_{t-i-1}^2 (\*x_t) } \\
        &\overset{(iii)}{\leq} \frac{\lambda^{-2} \sigma_{t-i-1}^2(\*x_{t-i})}{\lambda^{-2} \sigma_{t-i-1}^2(\*x_{t-i}) + 1} \\
        &\overset{(iv)}{\leq} \ln\rbr{1 + \lambda^{-2} \sigma_{t-i-1}^2(\*x_{t-i})}.
    \end{align}
    The inequalities follow from
    (i) Eq.~\eqref{eq:post_var_onestep_LB},
    (ii) $h(ab) \geq ah(b)$ for all $a \in [0, 1]$,
    (iii)~$h\rbr{ Q \lambda^{-2} \sigma_{t-i-1}^2 (\*x_t) } \leq 1$,
    and (iv) $\frac{a}{1 + a} \leq \ln(1 + a)$ for all $a \geq 0$.

    For the first term of Eq.~\eqref{eq:count_telescoping}, we see that
    \begin{align}
        \sum_{t \in \widetilde{\cT}} h\rbr{ Q \lambda^{-2} \sigma_{t-1}^2 (\*x_t) }
        &\leq Q \sum_{t \in \widetilde{\cT}} h\rbr{ \lambda^{-2} \sigma_{t-1}^2 (\*x_t) } \\
        &\leq 4Q \sum_{t \in \widetilde{\cT}} \frac{1}{2} \ln \rbr{1 + \lambda^{-2} \sigma_{t-1}^2 (\*x_t) } \\
        &\leq 4Q \gamma_{|\widetilde{\cT}|}(\lambda^2) \leq 4Q \gamma_{T}(\lambda^2),
    \end{align}
    where the first and second inequalities hold because $h(ab) \leq ah(b)$ for all $a \geq 1$ and $h(a) \leq 2 \ln(1 + a)$ for all $a \geq 0$, respectively.
    For the second term of Eq.~\eqref{eq:count_telescoping}, we have
    \begin{align}
        &\sum_{t \in \widetilde{\cT}} \sum_{i=1}^Q \cbr{ h\rbr{ Q \lambda^{-2} \sigma_{t-i-1}^2 (\*x_t) } - h\rbr{ Q \lambda^{-2} \sigma_{t-i}^2 (\*x_t) }} \\
        &= \sum_{i=1}^Q \sum_{t \in \widetilde{\cT}} \cbr{ h\rbr{ Q \lambda^{-2} \sigma_{t-i-1}^2 (\*x_t) } - h\rbr{ Q \lambda^{-2} \sigma_{t-i}^2 (\*x_t) }} \\
        &\leq \sum_{i=1}^Q \sum_{t \in \widetilde{\cT}} \ln\rbr{1 + \lambda^{-2} \sigma_{t-i-1}^2(\*x_{t-i})}.
    \end{align}
    From the monotonicity of the posterior variance, we can see that $\sum_{t \in \widetilde{\cT}} \ln\rbr{1 + \lambda^{-2} \sigma_{t-i-1}^2(\*x_{t-i})} \leq 2\gamma_{|\widetilde{\cT}|}(\lambda^2)$.
    Hence, we have
    \begin{align}
        \sum_{t \in \widetilde{\cT}} \sum_{i=1}^Q \cbr{ h\rbr{ Q \lambda^{-2} \sigma_{t-i-1}^2 (\*x_t) } - h\rbr{ Q \lambda^{-2} \sigma_{t-i}^2 (\*x_t) }} 
        &\leq 2Q \gamma_{|\widetilde{\cT}|}(\lambda^2) 
        \leq 2Q \gamma_{T}(\lambda^2).
    \end{align}
    Consequently, we obtain
    \begin{align}
        |\widetilde{\cT}| 
        &\leq 6Q \gamma_{|\widetilde{\cT}|}(\lambda^2) 
        \leq 6Q \gamma_{T}(\lambda^2),
    \end{align}
    and 
    \begin{align}
        |\cT| 
        &\leq 6Q \gamma_{|\cT|}(\lambda^2) + Q
        \leq 6Q \gamma_{T}(\lambda^2) + Q.
    \end{align}

    Second, we show $\sum_{t \in \cT^c} \sigma_{t - Q - 1}^2(\*x_t) \leq 8 \lambda^2 \gamma_T$.
    Let $\widetilde{\sigma}_{t-1}^2(\*x)$ be the posterior variance computed based on the dataset $\widetilde{\cD}_{t-1} = \cD_{t-Q-1} \cup \{ (\*x_i, y_i) \}_{i \in \cM_t}$, where $\cM_t = \cT^c \cap \{ t-Q, \dots, t-1\}$.
    If $\cM_t = \emptyset$, then $\sigma_{t-Q-1}^2(\*x_t) = \widetilde{\sigma}_{t-1}^2(\*x_t)$.
    Moreover, by applying Lemma~\ref{lemPoteriorVarianceInequality}, we have
    \begin{align}
        \sigma_{t-Q-1}^2(\*x_t)
        &\leq \rbr{1 + \lambda^{-2} \sum_{i \in \cM_t } \sigma_{t-Q-1}^2(\*x_i)} \widetilde{\sigma}_{t-1}^2(\*x_t) \\
        &\leq 2 \widetilde{\sigma}_{t-1}^2(\*x_t),
    \end{align}
    where the second inequality follows from $|\cM_t| \leq Q$ and
    $\sigma_{i-Q-1}^2(\*x_i) \leq \frac{\lambda^2}{Q}$
    due to $i < t$ and $i \in \cT^c$.
    Furthermore, define $\check{\sigma}_{t-1}^2 (\*x)$ as the posterior variance computed based on the dataset $\check{\cD}_{t-1} = \{ (\*x_i, y_i) \}_{i \in \rbr{\cT^c \cap [t-1]}}$.
    Since $\check{\cD}_{t-1} \subseteq \widetilde{\cD}_{t-1}$, we have $\widetilde{\sigma}_{t-1}^2(\*x) \leq \check{\sigma}_{t-1}^2 (\*x)$.
    Thus, we have
    \begin{align}
        \sum_{t \in \cT^c} \sigma_{t-Q-1}^2(\*x_t)
        &\leq 
        2 \sum_{t \in \cT^c} \widetilde{\sigma}_{t-1}^2(\*x_t) \\
        &=
        2 \lambda^2 \sum_{t \in \cT^c} \lambda^{-2} \widetilde{\sigma}_{t-1}^2(\*x_t) \\
        &= 
        2 \lambda^2 \sum_{t \in \cT^c} h\rbr{ \lambda^{-2} \widetilde{\sigma}_{t-1}^2(\*x_t) }\\
        &\leq 
        8 \lambda^2 \sum_{t \in \cT^c} \frac{1}{2} \ln \rbr{1 + \lambda^{-2} \widetilde{\sigma}_{t-1}^2(\*x_t) }\\
        &\leq 
        8 \lambda^2 \sum_{t \in \cT^c} \frac{1}{2} \ln \rbr{1 + \lambda^{-2} \check{\sigma}_{t-1}^2(\*x_t) }\\
        &\leq 8 \lambda^2 \gamma_T,
    \end{align}
    where the second equality holds because 
    $\widetilde{\sigma}_{t-1}^2(\*x_t) \leq \sigma_{t-Q-1}^2 (\*x_t) \leq \lambda^2 / Q \leq \lambda^2$ due to $t \in \cT^c$.
    The second, third, and fourth inequalities follow from
    $h(a) \leq 2 \ln(1 + a)$ for all $a \geq 0$,
    $\widetilde{\sigma}_{t-1}^2(\*x) \leq \check{\sigma}_{t-1}^2 (\*x)$,
    and the MIG upper bound, respectively.
\end{proof}

\subsection{Proof for noisy setting}
\label{sec:proof_noisy}

\TheoCumulativeRegretBoundNoisy*
\begin{proof}
    From Lemma~\ref{lemInflatedEllipticalPotentialLemmaRevised}, we can derive
    \begin{align}
        \sum_{t = 1}^T \min \left\{2B, c_t \sigma_{t-Q-1} (\*x_t)\right\}
        &\leq 2B |\cT| + c_T \sum_{t \in \cT^c} \sigma_{t-Q-1}(\*x_t) \\
        &\leq 2B |\cT| + 2 c_T \sqrt{2 \lambda^2 T \gamma_T},
    \end{align}
    where the second inequality holds from the Cauchy--Schwartz inequality.
    %
    %
    Furthermore, we see that $|\cT| \leq \overline{T}^{(Q)}$ since $|\cT| \leq 6Q \gamma_{|\cT|} + Q$ from Lemma~\ref{lemInflatedEllipticalPotentialLemmaRevised}.
    By substituting the above equation and $\sum_{t=1}^{\infty} 1/ t^2 = \pi^2 / 6$ into the sum of upper bounds derived in Lemma~\ref{lem:GeneralRegretUpperBound}, we can obtain the desired result.

    Next, we show the simple regret upper bound.
    Let the event $E^{(f)}_t$ be
    \begin{align}
        E^{(f)}_t = \{\forall t \in [T], \forall \*x \in \cX, |f(\*x) - \mu_{t-1}(\*x)| 
        \leq \beta_t^{1/2}\bigl(1 / t^2\bigr) \sigma_{t-1}(\*x) \}.
    \end{align}
    In addition, define ${\rm LCB}_t(\*x) \coloneqq \mu_{t-1}(\*x) - \beta_t^{1/2}\bigl(1 / t^2\bigr) \sigma_{t-1}(\*x)$.
    Then, since $\hat{\*x}_T = \argmax_{\*x \in \cX} {\rm LCB}_T(\*x)$, we can obtain the upper bound as follows:
    \begin{align}
        \EE[r_T]
        &= \EE\left[\min \left\{2B,f(\*x^*) - f(\hat{\*x}_T)\right\}\right] \\
        &\leq \EE\left[ \min \left\{2B, f(\*x^*) - f(\hat{\*x}_T) \right\} \mid E^{(f)}_T\right] \Pr\left( E^{(f)}_T \right) + 2B \Pr\left( {E^{(f)}_T}^c \right)\\
        &\leq \EE\left[ \min \left\{2B, f(\*x^*) - {\rm LCB}_T(\hat{\*x}_T) \right\} \mid E^{(f)}_T\right] \Pr\left( E^{(f)}_T \right) + \frac{2B}{T^2} \\
        &\leq \frac{1}{T} \sum_{t=1}^T \EE\left[ \min \left\{2B, f(\*x^*) - {\rm LCB}_T(\*x_t) \right\} \mid E^{(f)}_T\right] \Pr\left( E^{(f)}_T \right) + \frac{2B}{T^2} \\
        &\leq \frac{1}{T} \sum_{t=1}^T \EE\left[\min \left\{2B, f(\*x^*) - f(\*x_t) + f(\*x_t) - {\rm LCB}_T(\*x_t) \right\} \mid E^{(f)}_T\right] \Pr\left( E^{(f)}_T \right) + \frac{2B}{T^2} \\
        &\leq \frac{1}{T} \sum_{t=1}^T \EE\left[f(\*x^*) - f(\*x_t) + \min \left\{2B, 2 \beta_T^{1/2}\bigl(1 / T^2\bigr) \sigma_{t-1}(\*x_t) \bigr) \right\} \mid E^{(f)}_T\right] \Pr\left( E^{(f)}_T \right) + \frac{2B}{T^2},
    \end{align}
    where 
    (i) the first inequality holds due to $\max_{\*x \in \cX} |f(\*x)| \leq B$,
    (ii) the second inequality holds due to the definition of $E^{(f)}_T$ and $\Pr\left( {E^{(f)}_T}^c \right) \leq 1/ T^2$ from Lemma~\ref{lem:CI_noisy},
    (iii) the third inequality holds since for all $t \in [T]$, ${\rm LCB}_t(\hat{\*x}_T) \geq {\rm LCB}_t(\*x_t)$,
    and (iv) the fourth inequality holds since $E^{(f)}_T$ is true and $\sigma_{t}(\*x) \geq \sigma_{T}(\*x)$ for all $t \leq T$ and $\*x \in \cX$.
    Hence, we can finally obtain
    \begin{align}
        \EE[r_T]
        &\leq 
        \frac{\EE[R_T] + 2 \EE\left[\sum_{t=1}^T \min\left\{2B, \beta_T^{1/2}\bigl(1 / T^2\bigr) \sigma_{t-1}(\*x_t) \right\} \right]}{T} 
        + \frac{2B}{T^2}.
    \end{align}
    Then, by the same proof as that of cumulative regret, we can obtain the desired order of the simple regret.

    Finally, we show the order of $\overline{T}^{(Q)}$.
    For the case of SE kernels, it suffices to show
    \begin{align}
        &Q \ln^{d+1} Q = \Omega\rbr{ Q \ln^{d+1}\rbr{Q \ln^{d+1} Q} } \\
        &\Leftrightarrow
        C \ln^{d+1} Q > \ln^{d+1}\rbr{Q} + \rbr{(d + 1) \ln \ln Q}^{d+1}
    \end{align}
    where $C$ is some absolute constant.
    Since we can choose an absolute constant $c$ such that $c \ln^{d+1}(Q) > \rbr{(d + 1) \ln \ln Q}^{d+1}$, we can choose $C = 1 + c$.
    For the case of Mat\'ern kernels, by setting $\overline{T}^{(Q)} = \widetilde{O}\rbr{Q^{\frac{2\nu + d}{2\nu}} }$ and we see that $Q \gamma_{\overline{T}^{(Q)}} = \widetilde{O}\rbr{Q^{\frac{2\nu + d}{2\nu}} }$.
    Therefore, we can see that $\overline{T}^{(Q)}$ and $Q \gamma_{\overline{T}^{(Q)}}$ have the same order with respect to $Q$, which is the desired result.
\end{proof}

\subsection{Proof for noiseless setting}
\label{sec:proof_noiseless}

\begin{lemma}[Formal version of Lemma~\ref{lemPosteriorSTDUpperBoundNoiseless_informal}, Posterior standard deviation upper bound for SE and Mat\'ern kernel in the noiseless and batched setting]
    \label{lemPosteriorSTDUpperBoundNoiseless}
    Fix any compact input domain $\cX \subset \RR^d$, and kernel function $k: \cX \times \cX \rightarrow \RR$ that satisfies $k(\*x, \*x) \leq 1$ for all $\*x \in \cX$. 
    Furthermore, let $C_{\rm SE}$, $C_{\rm Mat}$, $\underline{\lambda}_{\rm SE}$, $\underline{\lambda}_{\rm Mat} > 0$, $\underline{T}_{\rm SE}$, $\underline{T}_{\rm Mat} \geq 2$ be the constants that satisfies $\forall \lambda \in (0, \underline{\lambda}_{\rm SE}], \forall t \geq \underline{T}_{\rm SE}, \gamma_t(\lambda^2) \leq C_{\rm SE} (\ln  (t/\lambda^2))^{d+1}$
    and $\forall \lambda \in (0, \underline{\lambda}_{\rm Mat}], \forall t \geq \underline{T}_{\rm Mat}, \gamma_t(\lambda^2) \leq C_{\rm Mat} (t/\lambda^2)^{\frac{d}{2\nu+d}}(\ln  (t/\lambda^2))^{\frac{2\nu}{2\nu+d}}$ for SE kernels and Mat\'ern kernels, respectively.
    Let $B, Q \geq 0$ and $\{ c_t \}_{t \in \NN}$ be monotonically non-decreasing sequence.
    Then, the following statements hold for any $T \in \NN$ and any input sequence $\*x_1, \ldots, \*x_T \in \cX$:
    \begin{itemize}
        \item For the SE kernels,
        \begin{align*}
            \min_{t \in [T]} \sigma_{t-Q-1}(\*x_t) &\leq 
            \begin{cases}
                1 ~~&\mathrm{if}~~ T < Q\overline{T}_{\rm SE}, \\
                \sqrt{2 T / Q} \exp\rbr{-\widetilde{C}_{\rm SE} (T / Q)^{\frac{1}{d+1}}} ~~&\mathrm{if}~~ T\geq Q \overline{T}_{\rm SE},
            \end{cases} \\
             \sum_{t=1}^T \min \cbr{2B, c_t \sigma_{t-Q-1}(\*x_t)} 
             &\leq 
             2BQ \overline{T}_{\rm SE} 
             + \sqrt{2} c_T Q (d+1) \rbr{\frac{\widetilde{C}_{\rm SE}}{2}}^{-\frac{3d+3}{2}} \Gamma\rbr{\frac{3d + 3}{2}},
        \end{align*}
        where $\widetilde{C}_{\rm SE} = (8C_{\rm SE})^{-\frac{1}{d+1}}$ and $\overline{T}_{\rm SE} = \max\{\underline{T}_{\rm SE}, \underline{T}^{(\lambda)}_{\rm SE}, \lceil (d+1)^{d+1}/\widetilde{C}_{\rm SE}^{d+1} \rceil + 1\}$ with $\underline{T}^{(\lambda)}_{\rm SE} = \min\{T \in \NN \mid \forall t \geq T, 2 t\exp(-\widetilde{C}_{\rm SE} t^{\frac{1}{d+1}}) \leq \underline{\lambda}_{\rm SE}^2\}$.
        \item For the Mat\'ern kernels with $\nu > 1/2$,
        \begin{align*}
            \min_{t \in [T]} \sigma_{t-Q-1} (\*x_t) &\leq \begin{cases}
                1 ~~&\mathrm{if}~~ T < Q \overline{T}_{\rm Mat}, \\
                \sqrt{2} \widetilde{C}_{\rm Mat}^{1/2} (T / Q)^{-\frac{\nu}{d}} \rbr{\ln  (T/Q)}^{\frac{\nu}{d}}  ~~&\mathrm{if}~~ T\geq Q\overline{T}_{\rm Mat},
            \end{cases} \\
            \sum_{t=1}^T \min \cbr{2B, c_t \sigma_{t-Q-1}(\*x_t)}
            &\leq \begin{cases}
                2BQ \overline{T}_{\rm Mat}  + \sqrt{2} c_T \widetilde{C}_{\rm Mat}^{1/2} \frac{d}{d-\nu} T^{\frac{d-\nu}{d}} Q^{\frac{\nu}{d}} (\ln  (T/Q))^{\frac{\nu}{d}} ~~&\mathrm{if}~~d > \nu, \\
                2BQ \overline{T}_{\rm Mat} + \sqrt{2} c_T \widetilde{C}_{\rm Mat}^{1/2} Q (\ln  (T/Q))^{2} / 2 ~~&\mathrm{if}~~d = \nu, \\
                2BQ \overline{T}_{\rm Mat} + \sqrt{2} c_T \widetilde{C}_{\rm Mat}^{1/2} \frac{Q\Gamma(\frac{\nu}{d} + 1)}{\rbr{\frac{\nu}{d}-1}^{\frac{\nu}{d}+1}} ~~&\mathrm{if}~~d < \nu,
            \end{cases}
        \end{align*}
        where $\widetilde{C}_{\rm Mat} = \max\cbr{1, \rbr{2 + \frac{2\nu}{d}}^{\frac{2\nu}{d}} (8C_{\rm Mat})^{1+\frac{2\nu}{d}}}$ and $\overline{T}_{\rm Mat} = \max\{4, \underline{T}_{\rm Mat}, \underline{T}^{(\lambda)}_{\rm Mat}\}$ with $\underline{T}^{(\lambda)}_{\rm Mat} = \min\{T \in \NN \mid \forall t \geq T, 2 \widetilde{C}_{\rm Mat} t^{-\frac{2\nu}{d}} \rbr{\ln  t}^{\frac{2\nu}{d}} \leq \underline{\lambda}_{\rm Mat}^2\}$.
    \end{itemize}
\end{lemma}
\begin{proof}
    In this proof, we describe the noise to calculate the posterior standard deviation explicitly as in $\sigma_{t-Q-1}(\*x_t; \lambda^2)$.
    Therefore, the posterior standard deviation written in this lemma is $\sigma_{t-1}(\*x_t) = \sigma_{t-1}(\*x_t ; 0)$.
    Note that if $\lambda \leq \tilde{\lambda}$, then $\sigma_{t-Q-1}(\*x_t; \lambda^2) \leq \sigma_{t-Q-1}(\*x_t; \tilde{\lambda}^2)$.
    Thus, to obtain the upper bound of posterior standard deviation in the noiseless setting, we can consider bounding as $\sigma_{t-Q-1}(\*x_t; 0) \leq \sigma_{t-Q-1}(\*x_t; \lambda^2_T)$ with some $\lambda^2_T > 0$.

    First, we show the variant of Lemma~\ref{lemInflatedEllipticalPotentialLemmaRevised} adapted from Lemma~\ref{lem:VakiliLem3}.
    Let $\lambda > 0$, $\cT = \cbr{t \in [T] \mid \sigma_{t-Q-1}(\*x_t) > \lambda }$ and $\cT_q = \cbr{t \in [T] \mid t \text{ mod } Q = q}$.
    Then, as with the proof of Lemma~\ref{lemInflatedEllipticalPotentialLemmaRevised}, for any $\lambda > 0$, we can obtain
    \begin{align}
        |\cT| 
        &\leq \sum_{t \in \cT} \min \{1, \lambda^{-2} \sigma_{t-Q-1}^2(\*x_t ; \lambda^2) \}\\
        &\leq \sum_{q=1}^Q \sum_{t \in \cT_q}
        \min \{1, \lambda^{-2} \sigma_{t-Q-1}^2(\*x_t ; \lambda^2) \}\\
        &\leq 4 \sum_{q=1}^Q \sum_{t \in \cT_q}\frac{1}{2} \ln \rbr{1 + \lambda^{-2} \sigma_{t-Q-1}^2(\*x_t ; \lambda^2)}\\
        &\leq 4 \sum_{q=1}^Q \gamma_{\lceil T / Q \rceil} (\lambda^2) \\
        &\leq 4 Q \gamma_{\lceil T / Q \rceil} (\lambda^2), \label{eq:EllipticalPotentialCountLargeVar}
    \end{align}
    where the second inequality holds because $\min\{1, a\} \leq 2 \ln(1 + a)$ for all $a \geq 0$.
    Then, we will control $\lambda_T$ to guarantee $|\cT^c| \geq 1$, that is, there are at least one index $t^\prime \in [T]$ such that $\sigma_{t^\prime-Q-1}(\*x_{t^\prime}) \leq \lambda_T$.

    For the SE kernels, we set $\lambda_t^2 = (2 t / Q) \exp(- \widetilde{C}_{\rm SE} (t/Q)^{\frac{1}{d+1}})$ and $\overline{T}_{\rm SE} \coloneqq \max\{\underline{T}_{\rm SE}, \underline{T}^{(\lambda)}_{\rm SE}, \lceil (d+1)^{d+1}/\widetilde{C}_{\rm SE}^{d+1} \rceil + 1\}$. 
    From the definition of $\lambda_t^2$, $\underline{T}_{\rm SE}$, and $\underline{T}_{\rm SE}^{(\lambda)}$, we have, for any $t \geq Q \overline{T}_{\rm SE}$, $\lambda_t^2 \leq \underline{\lambda}_{\rm SE}^2$ and
    \begin{align}
        \gamma_{\lceil t/Q \rceil}(\lambda_t^2) 
        &\leq C_{\rm SE} \sbr{\ln \rbr{ \left\lceil \frac{t}{Q} \right\rceil \lambda_t^{-2} }}^{d+1} \\
        &\leq C_{\rm SE} \sbr{\ln  \exp\rbr{\widetilde{C}_{\rm SE} (t/Q)^{\frac{1}{d+1}}}}^{d+1} \\
        &= \frac{C_{\rm SE} \widetilde{C}_{\rm SE}^{d+1} t}{Q},
    \end{align}
    where we use $\lceil t / Q \rceil \leq 2 t / Q$ for all $t \geq Q\overline{T}_{\rm SE} \geq Q$ for the second inequality.
    Furthermore,
    \begin{align}
        \frac{C_{\rm SE} \widetilde{C}_{\rm SE}^{d+1} t}{Q} \leq \frac{t-1}{4Q} 
        &\Leftrightarrow \widetilde{C}_{\rm SE}^{d+1} \leq \frac{t-1}{4 C_{\rm SE} t} \\
        &\Leftarrow \widetilde{C}_{\rm SE}^{d+1} \leq \frac{1}{8 C_{\rm SE}} \\
        &\Leftrightarrow \widetilde{C}_{\rm SE} \leq \rbr{\frac{1}{8 C_{\rm SE}}}^{\frac{1}{d+1}},
    \end{align}
    where the second line follows from the inequality $t - 1 \geq t/2$ for all $t \geq Q\overline{T}_{\rm SE} \geq 2$. 
    From the definition of $\widetilde{C}_{\rm SE}$, we conclude that 
    $\forall t \geq \overline{T}_{\rm SE}, \gamma_{\lceil t/Q \rceil}(\lambda_t^2) 
    \leq C_{\rm SE} \widetilde{C}_{\rm SE}^{d+1} t / Q
    \leq \frac{t-1}{4Q}$ 
    from the above inequalities.

    Since $|\cT| \leq 4Q \gamma_{\lceil T/Q \rceil}(\lambda_T^2) \leq T-1$ from Eq.~\eqref{eq:EllipticalPotentialCountLargeVar}, we can see that there exists $t^\prime \in [T]$ such that $\sigma_{t-Q-1}(\*x_t; \lambda_T^2) \leq \lambda_T$.
    Therefore, we can further obtain 
    $\min_{t \in [T]} \sigma_{t-Q-1}(\*x_t; 0) 
    \leq \min_{t \in [T]} \sigma_{t-Q-1}(\*x_t; \lambda_T^2) 
    \leq \lambda_T
    \leq \sqrt{ 2 T/Q } \exp(- \widetilde{C}_{\rm SE} (T/Q)^{\frac{1}
    {d+1}})$ for $T \geq Q \overline{T}_{\rm SE}$.
    For $T < Q \overline{T}_{\rm SE}$, it is obvious that $\sigma_{t-Q-1}(\*x_t; 0) \leq k(\*x_t, \*x_t) \leq 1$ from Assumption~\ref{assump:function}.

    Regarding the sum of the posterior standard deviations, by repeatedly picking the index $t^\prime$ that satisfies $\sigma_{t^\prime-Q-1}(\*x_{t^\prime}; \lambda_t^2) \leq \lambda_t$ for all $t \geq Q \overline{T}_{\rm SE}$ as with \citep{iwazaki2025gaussian}, we can see that
    \begin{align}
        \sum_{t=1}^T \min \cbr{2B, c_t\sigma_{t-Q-1}(\*x_t) }
        &= \sum_{t=1}^T \min \cbr{2B, c_t\sigma_{t-Q-1}(\*x_t ; 0) } \\
        &\leq 2BQ\overline{T}_{\rm SE} + c_T \sum_{t=Q\overline{T}_{\rm SE}}^{T} \lambda_t \\
        &\leq 2BQ \overline{T}_{\rm SE} + \sqrt{2} c_T \int_{Q\overline{T}_{\rm SE}-1}^T \sqrt{t/Q} \exp\rbr{-\frac{1}{2}\widetilde{C}_{\rm SE} (t/Q)^{\frac{1}{d+1}}} {\rm d}t\\
        &\leq 2BQ \overline{T}_{\rm SE} + \sqrt{2} c_T \int_{0}^T \sqrt{t/Q} \exp\rbr{-\frac{1}{2}\widetilde{C}_{\rm SE} (t/Q)^{\frac{1}{d+1}}} {\rm d}t,
    \end{align}
    where the second inequality follows from the fact that the function $g(t) \coloneqq (t/Q)\exp(-\widetilde{C}_{\rm SE} (t/Q)^{1/(d+1)})$ is non-increasing for $t \geq Q\overline{T}_{\rm SE} - 1$. 
    In fact, we have
    \begin{align}
        g^\prime(t) = \frac{1}{Q} \exp\rbr{-\widetilde{C}_{\rm SE} (t/Q)^{\frac{1}{d+1}}} \rbr{1 - \frac{\widetilde{C}_{\rm SE}}{d+1} (t/Q)^{\frac{1}{d+1}}},
    \end{align}
    which implies $g^\prime(t) \leq 0$ for $t \geq Q\overline{T}_{\rm SE} - 1 \geq Q (d+1)^{d+1}/\widetilde{C}_{\rm SE}^{d+1}$. 
    Let $u$ be $u = \widetilde{C}_{\rm SE} t^{\frac{1}{d+1}} / 2$.
    Then, we can obtain $t = \rbr{2u / \widetilde{C}_{\rm SE}}^{d+1}$ and ${\rm d}t / {\rm d}u = \rbr{2(d+1) / \widetilde{C}_{\rm SE}} \rbr{2u / \widetilde{C}_{\rm SE}}^{d}$.
    Therefore, we can obtain
    \begin{align}
        \int_{0}^T \sqrt{t/Q} \exp\rbr{-\frac{1}{2}\widetilde{C}_{\rm SE} (t/Q)^{\frac{1}{d+1}}} {\rm d}t
        &\leq Q \int_{0}^{\infty} \sqrt{t} \exp\rbr{-\frac{1}{2}\widetilde{C}_{\rm SE} t^{\frac{1}{d+1}}} {\rm d}t \\
        &= Q (d+1) \rbr{\frac{2}{\widetilde{C}_{\rm SE}}}^{\frac{3d+3}{2}} \int_{0}^{\infty} u^{\frac{3d+1}{2}} \exp(-u) {\rm d}u \\
        &= Q (d+1) \rbr{\frac{2}{\widetilde{C}_{\rm SE}}}^{\frac{3d+3}{2}} \Gamma\rbr{\frac{3d+2}{3}},
    \end{align}
    where the last equality uses the definition of the Gamma function.

    Next, for Mat\'ern kernels, we set $\lambda_t^2 = 2\widetilde{C}_{\rm Mat} Q^{\frac{2\nu}{d}} t^{-\frac{2\nu}{d}} (\ln  (t/Q))^{\frac{2\nu}{d}}$ and $\overline{T}_{\rm Mat} = \max\{4, \underline{T}_{\rm Mat}, \underline{T}^{(\lambda)}_{\rm Mat}\}$ with $\widetilde{C}_{\rm Mat} = \rbr{2 + \frac{2\nu}{d}}^{\frac{2\nu}{d}} (8C_{\rm Mat})^{1+\frac{2\nu}{d}}$. 
    Then since $\lceil t / Q \rceil \leq 2t / Q$ for all $t \geq Q \overline{T}_{\rm Mat} \geq Q$, we can obtain, for any $t \geq Q \overline{T}_{\rm Mat}$,
    \begin{align}
        \gamma_{\lceil t / Q \rceil }(\lambda_t^2) 
        &\leq C_{\rm Mat} \rbr{\frac{2t}{Q\lambda_t^2} }^{\frac{d}{2\nu+d}} \sbr{\ln \rbr{\frac{2t}{Q\lambda_t^2} }}^{\frac{2\nu}{2\nu+d}} \\
        &= C_{\rm Mat} \widetilde{C}_{\rm Mat}^{-\frac{d}{2\nu+d}} \frac{t}{Q} \rbr{\ln  \frac{t}{Q}}^{-\frac{2\nu}{2\nu+d}}  \sbr{\ln  \rbr{\widetilde{C}_{\rm Mat}^{-1} Q^{-\frac{2\nu + d}{d}} t^{\frac{2\nu + d}{d}} \rbr{\ln  \frac{t}{Q}}^{-\frac{2\nu}{d}} }}^{\frac{2\nu}{2\nu+d}} \\
        &= C_{\rm Mat} \widetilde{C}_{\rm Mat}^{-\frac{d}{2\nu+d}} \frac{t}{Q} \rbr{\ln  \frac{t}{Q}}^{-\frac{2\nu}{2\nu+d}}  \sbr{\ln  \rbr{\widetilde{C}_{\rm Mat}^{-1}} + \frac{d + 2\nu}{d} \rbr{\ln  \frac{t}{Q}} -\frac{2\nu}{d} (\ln  \ln  (t/Q)) }^{\frac{2\nu}{2\nu+d}} \\
        &\leq C_{\rm Mat} \widetilde{C}_{\rm Mat}^{-\frac{d}{2\nu+d}} \frac{t}{Q} \rbr{\ln  \frac{t}{Q}}^{-\frac{2\nu}{2\nu+d}}  \sbr{\frac{2d + 2\nu}{d} \rbr{\ln  \frac{t}{Q}}}^{\frac{2\nu}{2\nu+d}} \\
        &= C_{\rm Mat} \widetilde{C}_{\rm Mat}^{-\frac{d}{2\nu+d}} \frac{t}{Q} \rbr{\frac{2d + 2\nu}{d}}^{\frac{2\nu}{2\nu+d}}, 
    \end{align}
    where the fourth line holds because $\widetilde{C}_{\rm Mat} \geq 1 \Rightarrow \widetilde{C}_{\rm Mat} \geq Q/t \Leftrightarrow \ln  (\widetilde{C}_{\rm Mat}^{-1}) \leq \ln  (t/Q)$ for $t \geq Q$.
    Furthermore, 
    \begin{align}
        C_{\rm Mat} \widetilde{C}_{\rm Mat}^{-\frac{d}{2\nu+d}} \frac{t}{Q} \rbr{\frac{2d + 2\nu}{d}}^{\frac{2\nu}{2\nu+d}} \leq \frac{t-1}{4Q} 
        &\Leftrightarrow 4 C_{\rm Mat} \frac{t}{t-1} \rbr{\frac{2d + 2\nu}{d}}^{\frac{2\nu}{2\nu+d}} \leq \widetilde{C}_{\rm Mat}^{\frac{d}{2\nu+d}} \\
        &\Leftrightarrow \rbr{\frac{4 C_{\rm Mat} t}{t-1}}^{1 + \frac{2\nu}{d}} \rbr{2 + \frac{2\nu}{d}}^{\frac{2\nu}{d}} \leq \widetilde{C}_{\rm Mat} \\
        &\Leftarrow \rbr{8 C_{\rm Mat}}^{1 + \frac{2\nu}{d}} \rbr{2 + \frac{2\nu}{d}}^{\frac{2\nu}{d}} \leq \widetilde{C}_{\rm Mat}.
    \end{align}
    Combining the above inequalities, we can confirm $4Q \gamma_{\lceil t/ Q \rceil}(\lambda_t^2) \leq t-1$ for all $t \geq Q \overline{T}_{\rm Mat}$. 
    Therefore, as with the case of SE kernels, from Eq.~\eqref{eq:EllipticalPotentialCountLargeVar}, we can obtain 
    $\min_{t \in [T]} \sigma_{t-Q-1}(\*x_t; 0) 
    \leq \min_{t \in [T]} \sigma_{t-Q-1}(\*x_t; \lambda_T^2) 
    \leq \lambda_T
    \leq  \sqrt{2} \widetilde{C}_{\rm Mat}^{1/2} (T/Q)^{-\frac{\nu}{d}} (\ln  (T/Q))^{\frac{\nu}{d}}
    $ for $T \geq Q \overline{T}_{\rm Mat}$.

    Regarding the sum of posterior deviations, by repeatedly picking the index $t^\prime$ that satisfies $\sigma_{t^\prime-Q-1}(\*x_{t^\prime}; \lambda_t^2) \leq \lambda_t$ for all $t \geq Q \overline{T}_{\rm Mat}$ as with \citep{iwazaki2025gaussian}, we can see that
    \begin{align}
        \sum_{t=1}^T \min \cbr{2B, c_t \sigma_{t-Q-1}(\*x_t)}
        &= \sum_{t=1}^T \min \cbr{2B, c_t \sigma_{t-Q-1}(\*x_t; 0)} \\
        &\leq 2BQ\overline{T}_{\rm Mat} + c_T \sum_{t=Q\overline{T}_{\rm Mat}}^{T} \lambda_t \\
        &\leq 2BQ\overline{T}_{\rm Mat} + \sqrt{2} c_T \widetilde{C}_{\rm Mat}^{1/2} \int_{Q\overline{T}_{\rm Mat} - 1}^{T} \rbr{\frac{t}{Q}}^{-\frac{\nu}{d}} \rbr{\ln  \frac{t}{Q}}^{\frac{\nu}{d}} {\rm d}t \\
        &\leq 2BQ\overline{T}_{\rm Mat} + \sqrt{2} c_T \widetilde{C}_{\rm Mat}^{1/2} \int_{Q}^{T} \rbr{\frac{t}{Q}}^{-\frac{\nu}{d}} \rbr{\ln  \frac{t}{Q}}^{\frac{\nu}{d}} {\rm d}t,
    \end{align}
    where the second inequality follows from the fact that the function $g(t) \coloneqq (t/Q)^{-\frac{2\nu}{d}} (\ln  (t/Q))^{\frac{2\nu}{d}}$ is non-increasing for $t \geq Q\overline{T}_{\rm Mat} - 1 \geq 3Q > eQ$. Indeed, we have
    \begin{align}
        g^\prime(t) = \frac{1}{Q} \frac{2\nu}{d} (t/Q)^{-\frac{2\nu}{d}-1} (\ln  (t/Q))^{\frac{2\nu}{d}}\rbr{(\ln  (t/Q))^{-1} - 1},
    \end{align}
    which implies $g^\prime(t) \leq 0$ for $t \geq eQ$.
    The desired results are obtained by bounding the quantity 
    $\int_{Q}^{T} (t/Q)^{-\frac{\nu}{d}} (\ln  (t/Q))^{\frac{\nu}{d}} {\rm d}t
    = Q \int_{1}^{T/Q} t^{-\frac{\nu}{d}} (\ln  t)^{\frac{\nu}{d}} {\rm d}t
    $
    from above.
    When $d > \nu$, we have
    \begin{align}
        Q\int_{1}^{T/Q} t^{-\frac{\nu}{d}} (\ln  t)^{\frac{\nu}{d}} {\rm d}t
        &\leq Q (\ln  (T/Q))^{\frac{\nu}{d}} \int_1^{T/Q} t^{-\frac{\nu}{d}} {\rm d}t \\
        &= Q (\ln  (T/Q))^{\frac{\nu}{d}} \sbr{\frac{d}{d-\nu} t^{\frac{d-\nu}{d}}}_1^{T/Q} \\
        &\leq \frac{d}{d-\nu} T^{\frac{d-\nu}{d}} Q^{\frac{\nu}{d}} (\ln  (T/Q))^{\frac{\nu}{d}}.
    \end{align}
    When $d = \nu$,
    \begin{equation}
        Q\int_{1}^{T/Q} t^{-1} (\ln  t) \,{\rm d}t 
        = Q \left[ \frac{1}{2}(\ln  t)^2 \right]_1^{T/Q} 
        = \frac{Q}{2}(\ln  (T/Q))^2
    \end{equation}
    When $d < \nu$, we have
    \begin{align}
        Q\int_{1}^{T/Q} t^{-\frac{\nu}{d}} (\ln  t)^{\frac{\nu}{d}} {\rm d}t
        &= Q \int_{0}^{\ln  (T/Q)} e^{-\rbr{\frac{\nu}{d}-1}u} u^{\frac{\nu}{d}} {\rm d}u ~~~(\because u = \ln  t) \\
        &\leq Q \int_{0}^{\infty} e^{-\rbr{\frac{\nu}{d}-1}u} u^{\frac{\nu}{d}} {\rm d}u \\
        &= \frac{Q \Gamma(\frac{\nu}{d} + 1)}{\rbr{\frac{\nu}{d}-1}^{\frac{\nu}{d}+1}},
    \end{align}
    where the last equality follows from the standard property of the Gamma function: $\int_{0}^{\infty} e^{-\lambda u} u^b {\rm d}u = \Gamma(b+1)/\lambda^{b+1}$ for any $\lambda > 0$ and $b > -1$. 
\end{proof}

\paragraph{Note.}
We can obtain a result similar to Lemma~\ref{lemPosteriorSTDUpperBoundNoiseless_informal} also by using Lemma~\ref{lemInflatedEllipticalPotentialLemmaRevised} instead of Eq.~\eqref{eq:EllipticalPotentialCountLargeVar}, but in such a case, the order of the cumulative regret upper bound with respect to $Q$ becomes worse from $\Theta(Q)$ to $\omega(Q)$.
Here, we consider the case of SE kernels, but a similar discussion holds for Mat\'ern kernels.
If we use Lemma~\ref{lemInflatedEllipticalPotentialLemmaRevised} instead of Eq.~\eqref{eq:EllipticalPotentialCountLargeVar}, the necessary condition with respect to $T \geq Q \underline{T}_{\rm SE}$ is changed.
Specifically, if we naively use Lemma~\ref{lemInflatedEllipticalPotentialLemmaRevised}, then the necessary condition with respect to $T \geq Q\underline{T}_{\rm SE}$ can be tightened as $T \geq \underline{T}_{\rm SE}$.
This is because, to apply the MIG upper bound based on Eq.~\eqref{eq:EllipticalPotentialCountLargeVar}, $T / Q \geq \underline{T}_{\rm SE}$ is required, although only $T \geq \underline{T}_{\rm SE}$ is required for the case of Lemma~\ref{lemInflatedEllipticalPotentialLemmaRevised}.
On the other hand, the necessary condition with respect to $T \geq Q \underline{T}_{\rm SE}^{(\lambda)} \Leftrightarrow \lambda_T^2 \leq \underline{\lambda}_{\rm SE}$ is degraded.
This is because, if we naively use Lemma~\ref{lemInflatedEllipticalPotentialLemmaRevised}, then $\lambda_t^2$ is set as $\lambda_t^2 = 2 t \exp(- \widetilde{C}_{\rm SE} (t/Q)^{\frac{1}{d+1}})$, where the $1 / Q$ factor is removed.
Thus, the necessary condition with respect to $\lambda_T^2 \leq \underline{\lambda}_{\rm SE}$ becomes worse from $T = \Omega(Q)$ to $T = \omega(Q)$.
Thus, we employ the proof using Eq.~\eqref{eq:EllipticalPotentialCountLargeVar} to obtain the $O(Q)$ upper bound.

\corSimpleRegretUpperBoundNoiseless*
\begin{proof}
    From the settings of $1/N_t$ and $p_t^{(g)}$, the dominant term of the simple regret can be written as
    \begin{align}
        \EE[r_T]
        &\leq \EE \sbr{\min_{t \in [T]} f(\*x^*) - f(\*x_t)} \\
        &\leq \frac{1}{T} \EE \sbr{R_T}.
    \end{align}
    Thus, from Theorem~\ref{thm:CumulativeRegretUpperBoundNoiseless}, we can obtain the desired result.
\end{proof}

\end{document}

%% file: macros.tex
\def\*#1{\boldsymbol{#1}}

\theoremstyle{plain}

\newtheorem{theorem}{Theorem}[section]

\newtheorem{lemma}[theorem]{Lemma}

\theoremstyle{definition}
\newtheorem{definition}[theorem]{Definition}
\newtheorem{assumption}[theorem]{Assumption}
\theoremstyle{remark}
\newtheorem{remark}[theorem]{Remark}

\DeclareMathOperator*{\argmin}{arg\,min}
\DeclareMathOperator*{\argmax}{arg\,max}

\newcommand{\lw}[1]{\smash{\lower2.ex\hbox{#1}}}

\newcommand{\rbr}[1]{\left(#1\right)}
\newcommand{\sbr}[1]{\left[#1\right]}
\newcommand{\cbr}[1]{\left\{#1\right\}}

\newcommand{\RR}{\mathbb{R}}
\newcommand{\NN}{\mathbb{N}}

\newcommand{\EE}{\mathbb{E}}

\newcommand{\cD}{{\cal D}}

\newcommand{\cF}{{\cal F}}
\newcommand{\cG}{{\cal G}}
\newcommand{\cH}{{\cal H}}

\newcommand{\cM}{{\cal M}}
\newcommand{\cN}{{\cal N}}

\newcommand{\cP}{{\cal P}}

\newcommand{\cS}{{\cal S}}
\newcommand{\cT}{{\cal T}}

\newcommand{\cX}{{\cal X}}

%% file: main.bbl
\begin{thebibliography}{41}
\providecommand{\natexlab}[1]{#1}
\providecommand{\url}[1]{\texttt{#1}}
\expandafter\ifx\csname urlstyle\endcsname\relax
  \providecommand{\doi}[1]{doi: #1}\else
  \providecommand{\doi}{doi: \begingroup \urlstyle{rm}\Url}\fi

\bibitem[Abbasi-Yadkori(2013)]{abbasi2013online}
Yasin Abbasi-Yadkori.
\newblock \emph{Online learning for linearly parametrized control problems}.
\newblock PhD thesis, University of Alberta, 2013.

\bibitem[Bull(2011)]{bull2011convergence}
Adam~D Bull.
\newblock Convergence rates of efficient global optimization algorithms.
\newblock \emph{Journal of Machine Learning Research}, 12\penalty0 (10), 2011.

\bibitem[Cai and Scarlett(2021)]{cai2021on}
Xu~Cai and Jonathan Scarlett.
\newblock On lower bounds for standard and robust {G}aussian process bandit optimization.
\newblock In \emph{Proceedings of the 38th International Conference on Machine Learning}, volume 139, pages 1216--1226. PMLR, 2021.

\bibitem[Calandriello et~al.(2019)Calandriello, Carratino, Lazaric, Valko, and Rosasco]{calandriello2019gaussian}
Daniele Calandriello, Luigi Carratino, Alessandro Lazaric, Michal Valko, and Lorenzo Rosasco.
\newblock Gaussian process optimization with adaptive sketching: Scalable and no regret.
\newblock In \emph{Proceedings of the 32nd Conference on Learning Theory}, volume~99, pages 533--557. PMLR, 2019.

\bibitem[Chowdhury and Gopalan(2017)]{Chowdhury2017-on}
Sayak~Ray Chowdhury and Aditya Gopalan.
\newblock On kernelized multi-armed bandits.
\newblock In \emph{Proceedings of the 34th International Conference on Machine Learning}, volume~70, pages 844--853. PMLR, 2017.

\bibitem[Chowdhury and Gopalan(2019)]{chowdhury2019batch}
Sayak~Ray Chowdhury and Aditya Gopalan.
\newblock On batch {B}ayesian optimization.
\newblock \emph{arXiv:1911.01032}, 2019.

\bibitem[Contal et~al.(2013)Contal, Buffoni, Robicquet, and Vayatis]{Contal2013-Parallel}
Emile Contal, David Buffoni, Alexandre Robicquet, and Nicolas Vayatis.
\newblock Parallel {G}aussian process optimization with upper confidence bound and pure exploration.
\newblock In \emph{Proceedings of the 2013th European Conference on Machine Learning and Knowledge Discovery in Databases}, pages 225--240. Springer-Verlag, 2013.

\bibitem[De~Freitas et~al.(2012)De~Freitas, Smola, and Zoghi]{freitas2012exponential}
Nando De~Freitas, Alex~J. Smola, and Masrour Zoghi.
\newblock Exponential regret bounds for {G}aussian process bandits with deterministic observations.
\newblock In \emph{Proceedings of the 29th International Conference on Machine Learning}, page 955–962. Omnipress, 2012.

\bibitem[Desautels et~al.(2014)Desautels, Krause, and Burdick]{Desautels2014-Parallelizing}
Thomas Desautels, Andreas Krause, and Joel~W. Burdick.
\newblock Parallelizing exploration-exploitation tradeoffs in {Gaussian} process bandit optimization.
\newblock \emph{Journal of Machine Learning Research}, 15:\penalty0 4053--4103, 2014.

\bibitem[Flynn and Reeb(2025)]{flynn2025-tighter}
Hamish Flynn and David Reeb.
\newblock Tighter confidence bounds for sequential kernel regression.
\newblock In \emph{Proceedings of The 28th International Conference on Artificial Intelligence and Statistics}, volume 258, pages 3844--3852. PMLR, 2025.

\bibitem[Hern{\'a}ndez-Lobato et~al.(2017)Hern{\'a}ndez-Lobato, Requeima, Pyzer-Knapp, and Aspuru-Guzik]{hernandez-lobato2017parallel}
Jos{\'e}~Miguel Hern{\'a}ndez-Lobato, James Requeima, Edward~O. Pyzer-Knapp, and Al{\'a}n Aspuru-Guzik.
\newblock Parallel and distributed {T}hompson sampling for large-scale accelerated exploration of chemical space.
\newblock In \emph{Proceedings of the 34th International Conference on Machine Learning}, volume~70, pages 1470--1479. PMLR, 2017.

\bibitem[Iwazaki(2025{\natexlab{a}})]{iwazaki2025gaussian}
Shogo Iwazaki.
\newblock Gaussian process upper confidence bound achieves nearly-optimal regret in noise-free {G}aussian process bandits.
\newblock In \emph{Advances in Neural Information Processing Systems}, volume~38, pages 65863--65886. Curran Associates, Inc., 2025{\natexlab{a}}.

\bibitem[Iwazaki(2025{\natexlab{b}})]{iwazaki2025improved}
Shogo Iwazaki.
\newblock Improved regret bounds for {G}aussian process upper confidence bound in {B}ayesian optimization.
\newblock In \emph{Advances in Neural Information Processing Systems}, volume~38, pages 96922--96964. Curran Associates, Inc., 2025{\natexlab{b}}.

\bibitem[Iwazaki(2026)]{iwazaki2026tighter}
Shogo Iwazaki.
\newblock Tighter regret lower bound for {G}aussian process bandits with squared exponential kernel in hypersphere.
\newblock In \emph{Proceedings of the 43rd International Conference on Machine Learning}. PMLR, 2026.
\newblock To appear.

\bibitem[Iwazaki and Takeno(2025)]{iwazaki2025-improvedGPbandit}
Shogo Iwazaki and Shion Takeno.
\newblock Improved regret analysis in {G}aussian process bandits: Optimality for noiseless reward, {RKHS} norm, and non-stationary variance.
\newblock In \emph{Proceedings of the 42nd International Conference on Machine Learning}, volume 267, pages 26642--26672. PMLR, 2025.

\bibitem[Janz(2021)]{janz_2021}
David Janz.
\newblock \emph{Sequential decision making with feature-linear models}.
\newblock PhD thesis, University of Cambridge, 2021.

\bibitem[Janz et~al.(2020)Janz, Burt, and Gonzalez]{janz2020-bandit}
David Janz, David Burt, and Javier Gonzalez.
\newblock Bandit optimisation of functions in the {Matérn kernel RKHS}.
\newblock In \emph{Proceedings of the 23rd International Conference on Artificial Intelligence and Statistics}, volume 108, pages 2486--2495. PMLR, 2020.

\bibitem[Kanagawa et~al.(2018)Kanagawa, Hennig, Sejdinovic, and Sriperumbudur]{kanagawa2018gaussian}
Motonobu Kanagawa, Philipp Hennig, Dino Sejdinovic, and Bharath~K Sriperumbudur.
\newblock {G}aussian processes and kernel methods: A review on connections and equivalences.
\newblock \emph{arXiv:1807.02582}, 2018.

\bibitem[Kanagawa et~al.(2025)Kanagawa, Hennig, Sejdinovic, and Sriperumbudur]{kanagawa2025gaussian}
Motonobu Kanagawa, Philipp Hennig, Dino Sejdinovic, and Bharath~K. Sriperumbudur.
\newblock Gaussian processes and reproducing kernels: Connections and equivalences.
\newblock \emph{arXiv:2506.17366}, 2025.

\bibitem[Kandasamy et~al.(2018)Kandasamy, Krishnamurthy, Schneider, and P{\'o}czos]{Kandasamy2018-Parallelised}
Kirthevasan Kandasamy, Akshay Krishnamurthy, Jeff Schneider, and Barnabas P{\'o}czos.
\newblock Parallelised {B}ayesian optimisation via {T}hompson sampling.
\newblock In \emph{Proceedings of the 21st International Conference on Artificial Intelligence and Statistics}, volume~84, pages 133--142. PMLR, 2018.

\bibitem[Kim and Sanz-Alonso(2025)]{kim2025enhancing}
Hwanwoo Kim and Daniel Sanz-Alonso.
\newblock Enhancing {G}aussian process surrogates for optimization and posterior approximation via random exploration.
\newblock \emph{SIAM/ASA Journal on Uncertainty Quantification}, 13\penalty0 (3):\penalty0 1054--1084, 2025.

\bibitem[Li and Scarlett(2022)]{li2022gaussian}
Zihan Li and Jonathan Scarlett.
\newblock {G}aussian process bandit optimization with few batches.
\newblock In \emph{Proceedings of The 25th International Conference on Artificial Intelligence and Statistics}, volume 151, pages 92--107. PMLR, 2022.

\bibitem[Li and Scarlett(2024)]{li2024regret}
Zihan Li and Jonathan Scarlett.
\newblock Regret bounds for noise-free cascaded kernelized bandits.
\newblock \emph{Transactions on Machine Learning Research}, 2024.
\newblock URL \url{https://openreview.net/forum?id=oCfamUtecN}.

\bibitem[Lyu et~al.(2019)Lyu, Yuan, and Tsang]{lyu2019efficient}
Yueming Lyu, Yuan Yuan, and Ivor~W Tsang.
\newblock Efficient batch black-box optimization with deterministic regret bounds.
\newblock \emph{arXiv:1905.10041}, 2019.

\bibitem[Ma et~al.(2026)Ma, Chen, and Scarlett]{ma2026batched}
Chenkai Ma, Keqin Chen, and Jonathan Scarlett.
\newblock Batched kernelized bandits: Refinements and extensions.
\newblock \emph{arXiv:2603.12627}, 2026.

\bibitem[Nava et~al.(2022)Nava, Mutny, and Krause]{nava2022diversified}
Elvis Nava, Mojmir Mutny, and Andreas Krause.
\newblock Diversified sampling for batched {B}ayesian optimization with determinantal point processes.
\newblock In \emph{Proceedings of The 25th International Conference on Artificial Intelligence and Statistics}, volume 151, pages 7031--7054, 2022.

\bibitem[Pedregosa et~al.(2011)Pedregosa, Varoquaux, Gramfort, Michel, Thirion, Grisel, Blondel, Müller, Nothman, Louppe, Prettenhofer, Weiss, Dubourg, Vanderplas, Passos, Cournapeau, Brucher, Perrot, and Édouard Duchesnay]{scikit-learn}
Fabian Pedregosa, Gaël Varoquaux, Alexandre Gramfort, Vincent Michel, Bertrand Thirion, Olivier Grisel, Mathieu Blondel, Andreas Müller, Joel Nothman, Gilles Louppe, Peter Prettenhofer, Ron Weiss, Vincent Dubourg, Jake Vanderplas, Alexandre Passos, David Cournapeau, Matthieu Brucher, Matthieu Perrot, and Édouard Duchesnay.
\newblock Scikit-learn: Machine learning in {P}ython.
\newblock \emph{Journal of Machine Learning Research}, 12:\penalty0 2825--2830, 2011.

\bibitem[Rasmussen and Williams(2005)]{Rasmussen2005-Gaussian}
Carl~Edward Rasmussen and Christopher K.~I. Williams.
\newblock \emph{Gaussian Processes for Machine Learning (Adaptive Computation and Machine Learning)}.
\newblock The MIT Press, 2005.

\bibitem[Russo and Van~Roy(2014)]{Russo2014-learning}
Daniel Russo and Benjamin Van~Roy.
\newblock Learning to optimize via posterior sampling.
\newblock \emph{Mathematics of Operations Research}, 39\penalty0 (4):\penalty0 1221--1243, 2014.

\bibitem[Salgia et~al.(2024)Salgia, Vakili, and Zhao]{salgia2024random}
Sudeep Salgia, Sattar Vakili, and Qing Zhao.
\newblock Random exploration in {B}ayesian optimization: Order-optimal regret and computational efficiency.
\newblock In \emph{Proceedings of the 41st International Conference on Machine Learning}, volume 235, pages 43112--43141. PMLR, 2024.

\bibitem[Scarlett et~al.(2017)Scarlett, Bogunovic, and Cevher]{scarlett2017lower}
Jonathan Scarlett, Ilija Bogunovic, and Volkan Cevher.
\newblock Lower bounds on regret for noisy {G}aussian process bandit optimization.
\newblock In \emph{Proceedings of the 30th Conference on Learning Theory}, volume~65, pages 1723--1742. PMLR, 2017.

\bibitem[Srinivas et~al.(2010)Srinivas, Krause, Kakade, and Seeger]{Srinivas2010-Gaussian}
Niranjan Srinivas, Andreas Krause, Sham~M. Kakade, and Matthias~W. Seeger.
\newblock Gaussian process optimization in the bandit setting: No regret and experimental design.
\newblock In \emph{Proceedings of the 27th International Conference on Machine Learning}, pages 1015--1022. Omnipress, 2010.

\bibitem[Sugiura et~al.(2026)Sugiura, Takeuchi, and Takeno]{sugiura2026randomized}
Shuhei Sugiura, Ichiro Takeuchi, and Shion Takeno.
\newblock Randomized kriging believer for parallel {B}ayesian optimization with regret bounds.
\newblock \emph{arXiv:2603.01470}, 2026.

\bibitem[Vakili(2022)]{vakili2022open}
Sattar Vakili.
\newblock Open problem: {R}egret bounds for noise-free kernel-based bandits.
\newblock In \emph{Proceedings of the 35th Conference on Learning Theory}, volume 178, pages 5624--5629. PMLR, 2022.

\bibitem[Vakili et~al.(2021{\natexlab{a}})Vakili, Bouziani, Jalali, Bernacchia, and Shiu]{vakili2021-optimal}
Sattar Vakili, Nacime Bouziani, Sepehr Jalali, Alberto Bernacchia, and Da-shan Shiu.
\newblock Optimal order simple regret for {G}aussian process bandits.
\newblock In \emph{Advances in Neural Information Processing Systems}, volume~34, pages 21202--21215. Curran Associates, Inc., 2021{\natexlab{a}}.

\bibitem[Vakili et~al.(2021{\natexlab{b}})Vakili, Khezeli, and Picheny]{vakili2021-information}
Sattar Vakili, Kia Khezeli, and Victor Picheny.
\newblock On information gain and regret bounds in {G}aussian process bandits.
\newblock In \emph{Proceedings of The 24th International Conference on Artificial Intelligence and Statistics}, volume 130, pages 82--90. PMLR, 2021{\natexlab{b}}.

\bibitem[Vakili et~al.(2021{\natexlab{c}})Vakili, Moss, Artemev, Dutordoir, and Picheny]{vakili2021-scalable}
Sattar Vakili, Henry Moss, Artem Artemev, Vincent Dutordoir, and Victor Picheny.
\newblock Scalable {T}hompson sampling using sparse {G}aussian process models.
\newblock In \emph{Advances in Neural Information Processing Systems}, volume~34, pages 5631--5643. Curran Associates, Inc., 2021{\natexlab{c}}.

\bibitem[Vakili et~al.(2022)Vakili, Scarlett, Shiu, and Bernacchia]{vakili2022improved}
Sattar Vakili, Jonathan Scarlett, Da-Shan Shiu, and Alberto Bernacchia.
\newblock Improved convergence rates for sparse approximation methods in kernel-based learning.
\newblock In \emph{Proceedings of the 39th International Conference on Machine Learning}, volume 162, pages 21960--21983. PMLR, 2022.

\bibitem[Vakili et~al.(2023)Vakili, Ahmed, Bernacchia, and Pike-Burke]{vakili2023delayed}
Sattar Vakili, Danyal Ahmed, Alberto Bernacchia, and Ciara Pike-Burke.
\newblock Delayed feedback in kernel bandits.
\newblock In \emph{Proceedings of the 40th International Conference on Machine Learning}, volume 202, pages 34779--34792. PMLR, 2023.

\bibitem[Valko et~al.(2013)Valko, Korda, Munos, Flaounas, and Cristianini]{valko2013finite}
Michal Valko, Nathan Korda, R\'{e}mi Munos, Ilias Flaounas, and Nello Cristianini.
\newblock Finite-time analysis of kernelised contextual bandits.
\newblock In \emph{Proceedings of the 29th Conference on Uncertainty in Artificial Intelligence}, UAI'13, page 654–663. AUAI Press, 2013.

\bibitem[Verma et~al.(2022)Verma, Dai, and Low]{verma2022bayesian}
Arun Verma, Zhongxiang Dai, and Bryan Kian~Hsiang Low.
\newblock {B}ayesian optimization under stochastic delayed feedback.
\newblock In \emph{Proceedings of the 39th International Conference on Machine Learning}, volume 162, pages 22145--22167. PMLR, 2022.

\end{thebibliography}
